\documentclass[10pt]{article}
\usepackage{makecell}
\usepackage[table]{xcolor}
\usepackage{algorithm}
\usepackage{algpseudocode}
\usepackage{float}
\usepackage[
    a4paper,
    left=1.25in,
    right=1.26in,
    top=1.15in,
    bottom=1.15in
]{geometry}
\usepackage[T1]{fontenc}
\usepackage{lmodern}
\usepackage{microtype}
\usepackage{natbib}
\usepackage{amsmath,amssymb,mathtools,amsthm}

\usepackage{graphicx}
\usepackage{subcaption}
\usepackage{booktabs}
\usepackage{tabularx}
\usepackage{makecell}

\usepackage{pifont}
\usepackage{xcolor}
\usepackage{xfrac}

\usepackage{setspace}

\usepackage{hyperref}
\usepackage[capitalize,noabbrev]{cleveref}
\hypersetup{
    colorlinks=true,
    linkcolor=blue,
    citecolor=blue,
    urlcolor=blue
}

\usepackage[textsize=tiny]{todonotes}

\newtheorem{theorem}{Theorem}[section]

\newtheorem{lemma}{Lemma}[section]
\theoremstyle{remark}
\newtheorem{remark}[theorem]{Remark}

\def\reals{\mathbb{R}}

\newcommand{\supp}{\textbf{Support}\{\D\}}
\newcommand{\sigmal}{\sigma_L}

\newcommand{\tsigma}{\Tilde{\sigma}}
\newcommand{\TD}{\Tilde{D}}
\newcommand{\alg}{$\mu^2$-SGD~}

\newcommand{\A}{\mathcal{A}}
\newcommand{\B}{\mathcal{B}}
\newcommand{\D}{\mathcal{D}}

\newcommand{\Sync}{\mathcal{S}}

\newcommand{\ba}{\mathbf{a}}
\newcommand{\bb}{\mathbf{b}}

\newcommand{\bd}{\mathbf{d}}

\newcommand{\bg}{\mathbf{g}}
\newcommand{\tbg}{\Tilde{\bg}}

\newcommand{\bv}{\mathbf{v}}
\newcommand{\tbv}{\Tilde{\bv}}

\newcommand{\bx}{\mathbf{x}}
\newcommand{\bbx}{\Bar{\bx}}
\newcommand{\bxs}{\bx^\star}
\newcommand{\by}{\mathbf{y}}
\newcommand{\bz}{\mathbf{z}}
\newcommand{\beps}{\boldsymbol{\epsilon}}

\newcommand{\Exp}[1]{\mathbb{E}#1} 
\newcommand{\ExpB}[1]{\mathbb{E}\left[#1\right]} 
\newcommand{\ExpC}[2]{\mathbb{E}_{#1}\left[#2\right]} 
\newcommand{\floor}[1]{\left\lfloor#1\right\rfloor} 
\newcommand{\ceil}[1]{\left\lceil#1\right\rceil} 
\newcommand{\norm}[1]{\left\|#1\right\|} 
\newcommand{\normsq}[1]{\left\|#1\right\|^2} 
\newcommand{\dotprod}[2]{\left\langle#1,#2\right\rangle} 
\newcommand{\f}[1]{f\!\left(#1\right)} 
\newcommand{\df}[1]{\nabla f\!\left(#1\right)} 
\newcommand{\oo}[1]{\mathcal{O}\left(#1\right)} 
\newcommand{\om}[1]{\Omega\left(#1\right)} 
\newcommand{\ot}[1]{\Theta\left(#1\right)} 
\newcommand{\mini}[1]{\min\left\{#1\right\}} 
\newcommand{\K}[1]{K_{\rm #1}} 
\newcommand{\up}[1]{^{(#1)}} 
\def\al#1\eal{\begin{align}#1\end{align}} 
\def\als#1\eals{\begin{align*}#1\end{align*}} 
\newcommand{\non}{\nonumber\\} 


\title{Local MixVR: Breaking the Communication-Sample Dependence in Distributed Learning}

\author{
\normalsize
\begin{tabular}[t]{c}
Tehila Dahan\thanks{\texttt{t.dahan@campus.technion.ac.il}} \\
Technion \\
Haifa, Israel
\end{tabular}
\hspace{1.8em}
\begin{tabular}[t]{c}
Bassel Hamoud \\
Technion \\
Haifa, Israel
\end{tabular}
\hspace{1.8em}
\begin{tabular}[t]{c}
Roie Reshef \\
Technion \\
Haifa, Israel
\end{tabular}
\\[2.9em]
\begin{tabular}[t]{c}
\normalsize
Martin Jaggi \\
\normalsize
EPFL \\
\normalsize
Lausanne, Switzerland
\end{tabular}
\hspace{1.8em}
\begin{tabular}[t]{c}
\normalsize
Kfir Y. Levy \\
\normalsize
Technion \\
\normalsize
Haifa, Israel
\end{tabular}
}
\date{}

\begin{document}
\maketitle
\begingroup
\renewcommand\thefootnote{}
\endgroup
\vspace{-12pt}

\begin{abstract}
Communication overhead is a crucial bottleneck in scalable distributed learning.
While existing methods aim to efficiently utilize data points, such as Local SGD, Minibatch SGD, and their accelerated variants, they still exhibit communication-round complexity that scales with the total number of samples $N$.
In this paper, we introduce Local MixVR, a distributed framework that integrates local updates with variance-reduction techniques to mitigate local noise.
We show that Local MixVR is the first distributed method to eliminate the dependence of communication complexity on $N$, achieving a complexity that scales only with the number of workers $M$.
In common regimes where $M\leq\oo{N^{1/4}}$, Local MixVR outperforms the state-of-the-art Minibatch Accelerated SGD baseline, bridging a long-standing gap in distributed optimization and establishing a new paradigm for communication-efficient training.
\end{abstract}

\section{Introduction}
In the era of massive-scale machine learning \citep{verbraeken2020survey,zhao2023survey}, parallelizing training across multiple workers has become necessary.
However, communication overhead remains a fundamental bottleneck: as the number of workers increases, the cost of synchronizing updates can dominate overall training time.
This can be even more severe in poorly connected clusters, where limited bandwidth and high latency between nodes can significantly slow synchronization \citep{douillard2023diloco,jaghouar2024intellect}.
As a result, a central goal in distributed learning is to minimize the number of communication rounds while maintaining fast convergence.

A fundamental approach to parallel stochastic optimization is Minibatch SGD \citep{dekel2012optimal}, where $M$ workers compute stochastic gradients on minibatches of size $K$ and average them at each synchronization round to perform a global update.
This is equivalent to using an effective batch size of $MK$, which reduces the stochastic variance by a factor of $MK$.
Due to its simplicity and effectiveness, Minibatch SGD is widely used in distributed training \citep{brown2020language}.
However, increasing the effective batch size is beneficial only up to a point \citep{kaplan2020scaling}.
Once stochastic noise has been sufficiently reduced, further increases in batch size yield diminishing returns.
In many modern training regimes, the total batch size $MK$ enabled by many workers already exceeds the level needed for efficient learning.
Consequently, adding more workers may not substantially reduce the number of iterations to convergence, while it still increases synchronization and communication costs.
This can slow overall training despite the availability of additional compute.

This motivates distributed methods to use the available gradient information more effectively, rather than simply increasing the minibatch size.
Instead of using the extra gradients solely to reduce variance, one can also use them to make additional progress in optimization.
This is the idea behind Local SGD \citep{stichlocal, mangasarian1993backpropagation}: each worker performs $K$ local steps, then the server averages the models and sends the averaged model back to all workers.
This contrasts with Minibatch SGD, which performs only one global update per communication round.
More specifically, over $R$ rounds, Minibatch SGD performs $R$ optimization steps, whereas Local SGD performs $KR$ local updates per worker.
Thus, for the same number of processed gradients, Local SGD leverages this information to reduce gradient noise and make more progress along the optimization path.

\begin{table}[t] \centering \caption{\small Convergence rates and $R_{\min}$ for $\sigma=\oo{1}$ and $N=MKR$ total samples.} \label{tab:main-comparison} \footnotesize \renewcommand{\arraystretch}{1.15} \setlength{\tabcolsep}{4pt} \begin{tabular}{p{3.9cm} p{7cm} c} \toprule \textbf{Method} & \textbf{Rate} $[\oo{\cdot}]$ & $\boldsymbol{R_{\min}}~[\om{\cdot}]$ \\ \midrule \makecell[l]{MiniBatch SGD\\ \citep{dekel2012optimal}} & $\dfrac{1}{R}+\dfrac{\sigma}{\sqrt{MKR}}$ & $N^{1/2}$ \\ \makecell[l]{MiniBatch Accelerated SGD\\ \citep{lan2012optimal,dekel2012optimal}} & $\dfrac{1}{R^2}+\dfrac{\sigma}{\sqrt{MKR}}$ & \cellcolor{green!20}$N^{1/4}$ \\ \makecell[l]{Local SGD\\ \citep{khaled2020tighter}} & $\dfrac{1}{KR}+\dfrac{\sigma^{2/3}}{K^{1/3}R^{2/3}} +\dfrac{\sigma}{\sqrt{MKR}}$ & $MN^{1/2}$ \\ \makecell[l]{SCAFFOLD\\ \citep{karimireddy2020scaffold}} & $\dfrac{1}{R}+\dfrac{\sigma}{\sqrt{MKR}}$ & $N^{1/2}$ \\ \makecell[l]{Local Momentum\\ \citep{dahan2024slowcalsgd}} & $\dfrac{1}{KR}+\dfrac{\sigma^{1/2}}{K^{1/4}R} +\dfrac{1}{K^{1/3}R^{4/3}}+\dfrac{1}{R^2} +\dfrac{\sigma}{\sqrt{MKR}}$ & $M^{1/3}N^{1/3}$ \\ 
\makecell[l]{Accelerated Outer SGD\\ \citep{khaled2025understanding}} & $\dfrac{1}{KR^2} +\dfrac{\sigma^{2/3}}{M^{1/3}K^{1/3}R} +\dfrac{\sigma^{1/2}}{K^{1/4}R^{5/4}} +\dfrac{\sigma}{\sqrt{MKR}}$ & $M^{1/4}N^{1/4}$ \\
\midrule \makecell[l]{Local SGD lower bound\\ \citep{glasgow2022sharp}} & $\dfrac{1}{KR}+\dfrac{\sigma^{2/3}}{K^{1/3}R^{2/3}} +\dfrac{\sigma}{\sqrt{MKR}}$ & $MN^{1/2}$ \\ \makecell[l]{Local MixVR (ours)\\ \cref{thm:final}} & $\dfrac{1}{KR}+\dfrac{\sigma}{K^{\mathbf{1/2}}R} +\dfrac{\sigma}{\sqrt{MKR}}$ & \cellcolor{green!20}$\mathbf{M}$ \\ \bottomrule \end{tabular} \end{table}

This approach, however, introduces a new challenge: \emph{worker drift}.
Because each worker performs multiple updates with different samples, local models can diverge between synchronization rounds, especially when many local steps are taken.
Thus, although Local SGD seeks to extract more progress from each sample, its gains are limited by this accumulated drift.
Indeed, lower bounds \citep{glasgow2022sharp} show that Local SGD cannot require fewer communication rounds than Minibatch SGD to effectively utilize data points; see~\cref{tab:main-comparison}.

To assess the required number of communication rounds, one should ask, given a sample budget $N$, what is the minimum number of communication rounds such that additional optimization steps no longer improve the convergence rate?
The key trade-off is between the \emph{optimization error} and the \emph{statistical error}.
Since the statistical error is determined by the fixed sample budget $N$, there is a point beyond which further reducing the optimization error (by performing more optimization steps) does not improve the overall rate.
For example, Minibatch SGD has a convergence rate of order
\als
\underbrace{\dfrac{1}{R}}_{\text{optimization error}}+\underbrace{\dfrac{\sigma}{\sqrt{N}}}_{\text{statistical error}}
\eals
where $N:=MKR$ and $\sigma^2$ bounds the stochastic noise.
Thus, the statistical error dominates as long as
\vspace{-10pt}
\als
\frac{1}{R}\leq\oo{\frac{1}{\sqrt{N}}}
\eals
Consequently, once the number of communication rounds exceeds the critical threshold $R_{\min}=\om{\sqrt{N}}$, increasing $R$ no longer improves convergence.
Conversely, using fewer than $R_{\min}$ rounds worsens the rate, since the optimization error remains larger than the statistical error, which is the lower bound for first-order methods using $N$ samples \citep{nemirovskij1983problem, agarwal2009information}.
Thus, $R_{\min}$ is the minimum number of communication rounds required to retain the optimal convergence rate.
Although existing local methods aim to reduce this threshold (\cref{tab:main-comparison}), none provably outperform Minibatch Accelerated SGD (ASGD) \citep{lan2012optimal}.
In fact, existing lower bounds for smooth convex optimization suggest that improving upon Minibatch ASGD is impossible \citep{woodworth2021min}.
However, as \citet{woodworth2021min} note, the oracle in their lower-bound construction cannot be written as the gradient of a random smooth function and therefore does not capture the stochastic model in \cref{eq:Main}.
They leave open whether this structure can be exploited to obtain better convergence guarantees or whether the same lower bound can be proved using an oracle that satisfies this constraint.
Our work takes a step toward addressing this question by introducing a local method whose upper bounds \emph{surpass Minibatch ASGD in common regimes} under this stochastic model.
In particular, the communication complexity of existing methods, including Minibatch ASGD, \emph{grows with the sample size $N$}.
This dependence can be restrictive in common distributed settings, where datasets are often large, and communication is a primary bottleneck.
In this work, we introduce Local MixVR, a distributed framework that integrates three distinct variance-reduction techniques:
\begin{itemize}
\item
\textbf{Local double-momentum:}
We adopt a recent double-momentum approach \citep{dahan2025stochastic} to substantially reduce local stochastic variance and keep workers closely aligned.
\item
\textbf{Hybrid local-minibatching:}
Near the end of each communication round, we replace the remaining local steps with a single minibatch update of the same size to reduce worker divergence rather than amplify it by progressing locally.
\item
\textbf{Drift correction mechanism:}
We introduce a correction mechanism that mitigates bias arising from the transition from local to global parameters at synchronization steps.
\end{itemize}

\paragraph{Our Contributions:}
\begin{itemize}
\item
\textbf{Surpassing the Accelerated Baseline.}
We show that in the common regimes where $M\leq\oo{N^{1/4}}$, Local MixVR is the first algorithm to improve upon the state-of-the-art Minibatch ASGD baseline (\cref{tab:main-comparison}).
This regime arises naturally in practical training; for example, ImageNet-1K \citep{deng2009imagenet} has about $1.28\times10^6$ training images and is often trained with $M=8$ workers \citep{goyal2017accurate}, giving ${R^{\rm ASGD}_{\min}}/{R^{\rm MixVR}_{\min}}\approx N^{1/4}/M\approx4.2$.
Similarly, FineWeb \citep{penedo2024fineweb} has roughly $15\times10^{12}$ tokens and was trained with $M=64$ workers, yielding ${R^{\rm ASGD}_{\min}}/{R^{\rm MixVR}_{\min}}\approx30.7$.
Thus, Local MixVR can use fewer communication rounds than Minibatch ASGD while maintaining optimal performance.
\item
\textbf{Sample-Independent Communication Complexity.}
We show that Local MixVR is the first algorithm whose required number of communication rounds is independent of the sample size $N$ and depends only on the number of workers $M$, thereby yielding a stronger guarantee for settings with massive datasets.
\end{itemize}

\paragraph{Related Work.}
Local-update methods, including Local SGD and its variants \citep{koloskova2020unified,karimireddy2020mime,khaled2020tighter,stichlocal,gorbunov2021local,mishchenko2022proxskip,mishchenko2022proximal,pateltowards,cheng2023momentum,dahan2024slowcalsgd,yuan2020federated,reddiadaptive,mitra2021linear,zaccone2023communication}, reduce communication by performing multiple local steps between synchronization rounds.
Existing work has primarily focused on matching, and, in some regimes, improving upon, the Minibatch SGD rate (see \cref{tab:main-comparison}), while whether local computation can outperform Minibatch ASGD remains largely open \citep{nesterov2013introductory,lan2012optimal}.
Our work takes a step toward addressing this question by establishing improvements over Minibatch ASGD in common regimes.
A central obstacle to achieving such improvements is worker drift, which existing methods address through several complementary strategies.
SCAFFOLD \citep{karimireddy2020scaffold} and Mime \citep{karimireddy2020mime} use synchronized control variates to anchor local trajectories, whereas momentum-based methods \citep{cheng2023momentum,dahan2024slowcalsgd,yuan2020federated,zaccone2023communication} share momentum information across workers.
By incorporating global information from previous synchronization rounds, these mechanisms reduce local stochasticity and keep worker trajectories aligned.
A complementary strategy reduces gradient variance through minibatching \citep{dekel2012optimal}.
For example, Post-Local SGD \citep{lin2018don} first runs Minibatch SGD and then switches to local updates, while \citet{pateltowards} combine minibatches of size $\oo{K}$ with $\oo{K}$ local steps.
More advanced methods use momentum-based variance reduction, such as STORM \citep{cutkosky2019momentum}, which corrects the momentum estimate to track changes in the stochastic gradient and reduce its variance.
Incorporating this mechanism into local updates can improve stability and convergence \citep{karimireddy2020mime,pateltowards,cheng2023momentum}.

\paragraph{Settings.}
We consider stochastic convex optimization problems of the form
\al\label{eq:LossExpected}
\f{\bx}:=\ExpC{\bz\sim\D}{\f{\bx;\bz}},
\eal
where $\D$ is an unknown distribution.
We study a distributed setting in which $M$ workers jointly minimize $f$ over $R$ communication rounds, each performing $K$ local steps using independent samples.
The algorithm outputs $\bx_{\mathrm{out}}\in\reals^d$, and its performance is measured by the expected excess loss:
\als
\ExpB{\f{\bx_{\mathrm{out}}}}-\min_{\bx\in\reals^d}\f{\bx}.
\eals
At each local iteration $t$, worker $i$ samples $\bz_t\up{i}\sim\D$.
The worker then computes a stochastic gradient $\df{\bx_t\up{i};\bz_t\up{i}}$ that serves as an unbiased estimator, i.e., $\ExpB{\df{\bx_t\up{i};\bz_t\up{i}}\mid\bx_t\up{i}}=\df{\bx_t\up{i}}$.

\paragraph{Assumptions.}
We make the following assumptions for all $\bx,\by\in\reals^d$ and $\bz\in\supp$:

\textbf{Convexity}:~
The function $\f{\cdot;\bz}$ is convex, and there exists $\bxs\in\reals^d$ such that $\f{\bxs}\leq\f{\bx}$.
We will denote $D_1:=\norm{\bx_1-\bxs}$, the distance between the initial point and the optimal point.
\begin{flalign}\label{eq:bounded-variance}
&\text{\textbf{Bounded Variance}: There exists $\sigma>0$ such that } \Exp{\normsq{\df{\bx;\bz}-\df{\bx}}}\leq\sigma^2&
\end{flalign}
\begin{flalign}\label{eq:Main}
&\text{\textbf{Smoothness}: There exists $L>0$ such that } \norm{\df{\bx;\bz}-\df{\by;\bz}}\leq L\norm{\bx-\by}&
\end{flalign}
This implies that the expected loss $\f{\cdot}$ is $L$-smooth and that there exists $\sigmal\in[0,L]$ such that
\al\label{eq:sigmal}
\Exp{\normsq{(\df{\bx;\bz}-\df{\bx})-(\df{\by;\bz}-\df{\by})}}\leq\sigmal^2\normsq{\bx-\by}
\eal
The proof of \cref{eq:sigmal} is in \cref{app:smooth}.

\textbf{Notation:}
$\df{\bx;\bz}$ denotes the gradient with respect to $\bx$.
We use $\norm{\cdot}$ to denote the Euclidean norm.

\begin{algorithm}[t]
\caption{\alg}
\label{alg:mu2}
\begin{algorithmic}[1]
\small
\State \textbf{Input:} $\bbx_t,\bbx_{t-1},\bx_t,\bd_{t-1}$, stepsize $\eta_t$, momentum parameters $\beta_t,\gamma_t$
\State Sample $\bz_t\sim\D$, and compute $\bg_t\gets\df{\bbx_t;\bz_t}, \tbg_{t-1}\gets\df{\bbx_{t-1};\bz_t}$
\State $\bd_t\gets\bg_t+(1-\beta_t)\left(\bd_{t-1}-\tbg_{t-1}\right)$
\State $\bx_{t+1}\gets\bx_t-\eta_t\bd_t$
\State $\bbx_{t+1}\gets\gamma_t\bx_{t+1}+(1-\gamma_t)\bbx_t$
\State \textbf{return} $\bbx_{t+1},\bx_{t+1},\bd_t$
\end{algorithmic}
\end{algorithm}

\section{Local MixVR}
We introduce Local MixVR, a unified framework for mitigating worker drift in distributed learning with infrequent communication.
As discussed above, worker drift is a central limitation of local-update methods: when workers perform multiple stochastic steps between synchronization rounds, sampling noise causes their local models to diverge.
Controlling this drift typically requires the number of communication rounds to grow with the sample complexity.
Local MixVR breaks this dependence by combining complementary variance-reduction mechanisms that control drift at different stages of the algorithm.
Their joint effect stabilizes local trajectories and limits the stochastic error accumulated between synchronizations, enabling less frequent communication while preserving the optimal convergence rate.
Local MixVR consists of three main components:

\begin{enumerate}
\item[I.]
\textbf{Local Double-Momentum.}
During each local phase, every worker runs the \alg algorithm \citep{dahan2025stochastic} (\cref{alg:mu2}) to control drift and keep local trajectories aligned.
The algorithm combines two complementary momentum mechanisms:
\begin{itemize}
\item Anytime Averaging \citep{cutkosky2019anytime}
    maintains an averaged (momentum) sequence
    $\{\bbx_t\up{i}\}_{i=1}^M$ for the local iterates
    $\{\bx_t\up{i}\}_{i=1}^M$.
    Because the averaged iterates evolve more slowly, they remain close to the
    most recent synchronized global model parameters, stabilizing local trajectories and
    limiting worker drift.
\item
The STORM estimator \citep{cutkosky2019momentum}, which uses information from the previous local step to correct the current stochastic gradient estimate, thereby reducing its variance.
This produces less noisy update directions, limiting divergence among workers' local trajectories.
\end{itemize}
\item[II.]
\textbf{Budget Mixing: Local Progress vs. Minibatch Averaging.}
Local MixVR uses a mixing parameter $\alpha\in(0,1)$ to split the total budget of $K$ samples into two complementary roles: making progress in optimization and reducing stochastic noise before synchronization.
The parameter $\alpha$ controls this trade-off.
Specifically, Local MixVR allocates $\K{loc}=\floor{(1-\alpha)K}$ samples to local optimization steps, and reserves $\K{avg}=\ceil{\alpha K}$ samples for minibatch averaging.
Thus, for the same budget of $K$ samples, the local phase moves the optimization trajectory forward, while the averaging phase reduces stochastic noise before it is injected into the synchronized parameters.

\item[III.]
\textbf{Drift Correction Mechanism.}
Finally, Local MixVR corrects the gradient estimator at each synchronization to account for the drift accumulated during local updates.
At synchronization step $t$, the workers first average their local models $\{\bbx_t\up{i}\}_{i=1}^M$ to obtain
\al\label{eq:sync_params}
\bbx_t:=\frac{1}{M}\sum_{i=1}^M\bbx_t\up{i}.
\eal
Each worker $i$ then evaluates the gradient on the same minibatch $\B_t\up{i}$ at its previous local parameter $\bbx_{t-1}\up{i}$ and at the synchronized parameter $\bbx_t$.
The difference
\als
\df{\bbx_t;\B_t\up{i}}
-
\df{\bbx_{t-1}\up{i};\B_t\up{i}}
\eals
estimates the gradient change from the local parameter to the synchronized model.
Adding this term shifts the estimator to the synchronized point, correcting the mismatch accumulated during local updates.
This realigns the local estimators with the global model and allows workers to begin the next local phase with reduced drift and stochastic error.
\end{enumerate}

\paragraph{I. Local Double-Momentum.}
The first layer of our framework is a local double-momentum mechanism, denoted by \alg.
Following the Anytime-SGD framework of \citet{cutkosky2019anytime}, one momentum is applied directly to the model parameters.
In particular, stochastic gradients are evaluated not at the current iterate $\bx_t$, but at the momentum-averaged parameters $\bbx_t$:
\al\label{eq:anytime}
\bx_{t+1}=\bx_t-\eta_t\df{\bbx_t;\bz_t}, \qquad \bbx_{t+1}=\gamma_t\bx_{t+1}+(1-\gamma_t)\bbx_t
\eal
where $\bz_t\sim\D$.
With the choice $\gamma_t=\tfrac{2}{t+2}$, the averaged sequence evolves slowly:
\al
\label{eq:anytime-central}
\norm{\bbx_{t+1}-\bbx_t}=\gamma_t\norm{\bx_{t+1}-\bbx_t}\leq\oo{\frac{\norm{\bx_{t+1}-\bbx_t}}{t}}
\eal
In the distributed setting, this property ensures that local model trajectories remain closely aligned across workers, effectively curbing the drift induced by local updates \citep{dahan2024slowcalsgd}.
The second momentum mechanism is designed to reduce stochastic variance by correcting the bias inherent in standard momentum \citep{polyak1964some}.
This bias arises because standard momentum averages past stochastic gradients, which were evaluated at previous, and therefore stale, iterates.
The STORM estimator \citep{cutkosky2019momentum} addresses this by introducing a correction term:
\als
\bd_t=\underbrace{(1-\beta_t)\bd_{t-1}+\beta_t\df{\bx_t;\bz_t}}_{\text{standard momentum}}+\underbrace{(1-\beta_t)\left(\df{\bx_t;\bz_t}-\df{\bx_{t-1};\bz_t}\right)}_{\text{correction term}}
\eals
Equivalently, this update can be written in the compact form used throughout the paper:
\als
\bd_t=\df{\bx_t;\bz_t}+(1-\beta_t)\left(\bd_{t-1}-\df{\bx_{t-1};\bz_t}\right)
\eals
The correction term measures how the stochastic gradient changes when moving from $\bx_{t-1}$ to $\bx_t$, using the same sample $\bz_t$.
This adjustment compensates for the stale-gradient bias introduced by standard momentum.
As a result, $\bd_t$ is an (unconditionally) unbiased estimator of the true gradient $\df{\bx_t}$, that is, $\ExpB{\bd_t-\df{\bx_t}}=0$, though it is conditionally biased.
This unbiased estimator yields a stronger reduction in stochastic error than standard momentum \citep{cutkosky2019momentum}.

\textbf{\alg.}
As shown by \citet{dahan2025stochastic}, applying STORM to anytime-averaged iterates $\bbx_t$ can substantially reduce stochastic variance.
The resulting $\mu^2$-SGD update is
\als
&\bd_t=\df{\bbx_t;\bz_t}+(1-\beta_t)\left(\bd_{t-1}-\df{\bbx_{t-1};\bz_t}\right) \\
&\bx_{t+1}=\bx_t-\eta_t\bd_t, \qquad \bbx_{t+1}=\gamma_t\bx_{t+1}+(1-\gamma_t)\bbx_t
\eals
Choosing $\beta_t=1/t$ allows \alg to use the full gradient history, yielding a maximal $1/t$ reduction in stochastic variance (in a similar way to minibatch SGD)
\al
\label{eq:mu2-reduction}
\Exp{\normsq{\bd_t-\df{\bbx_t}}}\leq\oo{\frac{\tsigma^2}{t}}, \qquad \text{where}~\tsigma:=\ot{\sigma+\sigmal D_1}
\eal
Moreover, when averaging over $M$ workers, the error decreases further to $\oo{\frac{\tsigma^2}{Mt}}$.

\textbf{Local \alg.}
In Local MixVR, each worker $i$ applies \alg locally and computes its own \alg estimator during the local phase:
\als
&\bd_t\up{i}=\df{\bbx_t\up{i};\bz_t\up{i}}+(1-\beta_t)\left(\bd_{t-1}\up{i}-\df{\bbx_{t-1}\up{i};\bz_t\up{i}}\right) \\
&\bx_{t+1}\up{i}=\bx_t\up{i}-\eta_t\bd_t\up{i}, \qquad \bbx_{t+1}\up{i}=\gamma_t\bx_{t+1}\up{i}+(1-\gamma_t)\bbx_t\up{i}
\eals
The variance reduction achieved by \alg (\cref{eq:mu2-reduction}) can make the local updates more stable and better aligned across workers, thereby helping to mitigate worker drift.
However, moving from the minibatch setting of \citet{dahan2025stochastic} to a local-update regime introduces local bias.
This bias must therefore be controlled, as we show next with Local MixVR.

\begin{algorithm}[t]
\caption{Accumulation (minibatching) procedure and drift correction for worker $i$}
\label{alg:accumulation}
\begin{algorithmic}[1]
\small
\State \textbf{Input:} $\bbx_t,\bbx_{t-1}\up{i},\bd_{t-1}\up{i},\beta_t,\K{avg}$
\State $\bg_t\up{i}\gets0, \quad \tbg_{t-1}\up{i}\gets0$
\For{$k=1,\dots,\K{avg}$}
\Comment{accumulate gradients}
\State Sample $\bz_t\up{{i,k}}\sim\D$
\State Compute $\bg_t\up{i}\gets\bg_t\up{i}+\frac{1}{\K{avg}}\df{\bbx_t;\bz_{t,k}\up{{i}}}$
\State Compute $\tbg_{t-1}\up{i}\gets\tbg_{t-1}\up{i}+\frac{1}{\K{avg}}\df{\bbx_{t-1}\up{i};\bz_{t,k}\up{i}}$
\EndFor
\State $\bd_t\up{i}\gets\bg_t\up{i}+(1-\beta_t)\left(\bd_{t-1}\up{i}-\tbg_{t-1}\up{i}\right)$
\Comment{drift correction}
\State \textbf{return} $\bd_t\up{i}$
\end{algorithmic}
\end{algorithm}

\paragraph{II. Minibatch Averaging (Accumulation Phase).}
After the local-update phase, Local MixVR synchronizes the model parameters as in \cref{eq:sync_params}.
The STORM estimators, however, cannot be averaged directly.
Unlike the individual stochastic gradients used by Local SGD, each STORM estimator recursively incorporates a worker's previous stochastic gradients, which were evaluated at different local iterates, and therefore carries errors accumulated through local drift.
Directly averaging these estimators would propagate this error to the global estimator.
To prevent this, each synchronization step $t$ is followed by a gradient-accumulation phase (\cref{alg:accumulation}), which refines the estimators using gradients evaluated at the common synchronized model, reducing the variance carried over from the local phase.

For each worker $i$, let $\B_t\up{i}=\left\{\bz_{t,1}\up{i},\dots,\bz_{t,\K{avg}}\up{i}\right\}$ be a minibatch of size $\K{avg}$.
During the accumulation phase, worker $i$ evaluates two gradients using the same minibatch:
\als
\df{\bbx_t;\B_t\up{i}}:=&\frac{1}{\K{avg}}\sum_{k=1}^{\K{avg}}\df{\bbx_t;\bz_{t,k}\up{i}} , \qquad
\df{\bbx_{t-1}\up{i};\B_t\up{i}}:=&\frac{1}{\K{avg}}\sum_{k=1}^{\K{avg}}\df{\bbx_{t-1}\up{i};\bz_{t,k}\up{i}}
\eals
Using the same minibatch at both points yields a coupled gradient difference.
This coupling is important because the variance of the difference is controlled by the distance between the two iterates.
Moreover, averaging over $\K{avg}$ samples reduces this variance by a factor of $\K{avg}$.
Hence, under the random-function smoothness setting (\cref{eq:Main,eq:sigmal}), we have
\al\label{eq:smooth-drift}
\Exp{\normsq{\left(\df{\bbx_t;\B_t\up{i}}-\df{\bbx_t}\right)-\left(\df{\bbx_{t-1}\up{i};\B_t\up{i}}-\df{\bbx_{t-1}\up{i}}\right)}}\leq\frac{\sigmal^2}{\K{avg}}\Exp{\normsq{\bbx_t-\bbx_{t-1}\up{i}}}
\eal
We average over the workers and decompose this distance into two terms:
\als
\frac{1}{M}\sum_{i=1}^M\normsq{\bbx_t-\bbx_{t-1}\up{i}}=\underbrace{\normsq{\bbx_t-\bbx_{t-1}}}_{\textnormal{global step}}+\underbrace{\frac{1}{M}\sum_{i=1}^M\normsq{\bbx_{t-1}-\bbx_{t-1}\up{i}}}_{\textnormal{worker drift}}
\eals
The first term captures the global model’s most recent update, while the second measures the discrepancy between each local model and the synchronized model at time $t$.
As in \cref{eq:anytime-central}, the anytime mechanism reduces the distance between successive local iterates:
\al\label{eq:local-anytime}
\normsq{\bbx_t-\bbx_{t-1}}\leq\gamma_{t-1}^2\normsq{\bx_t-\bbx_{t-1}}\leq\oo{\frac{\TD^2}{t^2}}
\eal
Here, $\TD^2$ is defined in \cref{app:anal} and depends on the initial distance $D_1$ and the stochastic error across workers.
By \cref{lem:dist}, it bounds the distance of any local iterate $\bbx_t$ or $\bx_t$ from $\bxs$, and therefore also bounds distances between two generated iterates (see, e.g., \cref{eq:following-iter-anytime}).

Let $t_0$ denote the last synchronization time before $t$.
Since the local and global momentum sequences are initialized at the same point $\bbx_{t_0+1}$, this term cancels when subtracting their unrolled recursions.
Hence, for $\gamma_t=\frac{2}{t+2}$, $\bbx_t\up{i}-\bbx_t=\frac{2}{t(t+1)}\sum_{k=t_0+2}^t k\left(\bx_k\up{i}-\bx_k\right)$.
Applying \cref{lem:diff} yields
\als
\frac{1}{M}\sum_{i=1}^M\normsq{\bbx_t\up{i}-\bbx_t}\leq\oo{\frac{\K{loc}^2\TD^2}{t^2}}
\eals
Hence, before accumulation, the worker drift exceeds the local-iterate term by a factor of order $\K{loc}^2$.
The accumulation phase reduces this contribution by the additional factor $\K{avg}$ in \cref{eq:smooth-drift}.
Thus, with $\K{avg}=\ceil{\alpha K}$, $\K{loc}=\floor{(1-\alpha)K}$, and $\alpha\in(0,1)$, the accumulated worker-drift scales as
\al\label{eq:accumulated-drift}
\oo{\frac{\K{loc}^2}{\K{avg}}\frac{\TD^2}{t^2}}=\oo{\frac{K\TD^2}{t^2}}
\eal
for constant $\alpha$.
Within each synchronization round, the accumulated worker-drift error is of order $\oo{K\TD^2/t^2}$.
The local-iterate error also has the same scale: it is incurred over $\K{loc}$ local steps, each contributing $\approx \oo{\TD^2/t^2}$, yielding a total of $\approx \oo{K\TD^2/t^2}$.
Thus, after accumulation, both effects enter in the same order per round.
In this sense, the accumulation procedure keeps the variance due to worker drift comparable to the variance-reduction effect of the local iterates.

\paragraph{III. Drift Correction Mechanism.}
Prior to a synchronization step $t$, each worker's STORM estimator $\bd_{t-1}\up{i}$ is anchored to its local history, with an expectation $\ExpB{\bd_{t-1}\up{i}-\df{\bbx_{t-1}\up{i}}}=0$.
After synchronization, however, all workers move to the synchronized parameter $\bbx_t$.
Therefore, the estimator $\bd_{t}\up{i}$ should also be shifted from the stale local target $\df{\bbx_{t-1}\up{i}}$ to the new global target $\df{\bbx_{t}}$.
In other words, we want the corrected estimator to satisfy $\ExpB{\bd_{t}\up{i}-\df{\bbx_{t}}}=0$.
We achieve this by using the gradients from the accumulation phase as a bridge.
Specifically, we use $\df{\bbx_{t-1}\up{i};\B_t\up{i}}$ to subtract the error inherited from the local path and $\df{\bbx_t;\B_t\up{i}}$ to re-center the estimator toward the new synchronized parameter.
This synchronization correction is defined as:
\als
\Tilde{\bd}_t\up{i}=(1-\beta_t)\bd_{t-1}\up{i}+\beta_t\df{\bbx_t;\B_t\up{i}}+(1-\beta_t)\underbrace{\left(\df{\bbx_t;\B_t\up{i}}-\df{\bbx_{t-1}\up{i};\B_t\up{i}}\right)}_{\textnormal{drift correction}}
\eals
This correction removes the bias induced by the worker's local drift and shifts the estimator toward the synchronized global parameter.
Consequently, $\ExpB{\Tilde{\bd}_t\up{i}-\df{\bbx_t}}=0$, and after synchronizing the corrected estimators across all $M$ workers, the local corrected momentum $\bd_t\up{i}=\bd_t=\frac{1}{M}\sum_{i=1}^M\Tilde{\bd}_t\up{i}$ preserves the same property $\ExpB{\bd_t\up{i}-\df{\bbx_t}}=0$.
Moreover, as discussed above, because the correction uses minibatching during the accumulation phase, the drift error remains controlled.

\begin{algorithm}[t]
\caption{Local MixVR}
\label{alg:local-mu2}
\begin{algorithmic}[1]
\small
\State \textbf{Input:}
$\bbx_0\in\reals^d$; parameters $\{\eta_t>0,\beta_t,\gamma_t\in(0,1)\}_t$; workers $M$; rounds $R$; steps $K$; mixing ratio $\alpha\in(0,1)$.
Set $\K{loc}=\floor{(1-\alpha)K}$, $\K{avg}=\ceil{\alpha K}$, $\beta_1=1$, and $t=1$.
Let $\A$ be the accumulation and drift-correction procedure in \cref{alg:accumulation}, and $\textsc{\alg}$ the update in \cref{alg:mu2}.
\State \textbf{Initialize:}
$\bbx_1\up{i}=\bx_1\up{i}=\bbx_0$ and $\bd_0\up{i}=\mathbf{0}$ for all $i\in[M]$.
\For{$r=1,\dots,R$}
\State $\bbx_{t-1}\up{i}\gets\bbx_{t-1}$ for all $i\in[M]$
\For{$k=1,\dots,\K{loc}$}
\hfill\Comment{$t=(r-1)(\K{loc}+1)+k$}
\For{$i=1,\dots,M$ \textbf{in parallel}}
\hfill\Comment{local updates}
\State $(\bbx_{t+1}\up{i},\bx_{t+1}\up{i},\bd_t\up{i})\gets\textsc{\alg}(\bbx_t\up{i},\bbx_{t-1}\up{i},\bx_t\up{i},\bd_{t-1}\up{i},\eta_t,\beta_t,\gamma_t), \quad t\gets t+1$
\EndFor
\EndFor
\hfill\Comment{$t=r(\K{loc}+1)$}
\State $\displaystyle\bbx_t\gets\frac1M\sum_{i=1}^M\bbx_t\up{i}, \qquad \bx_t\gets\frac1M\sum_{i=1}^M\bx_t\up{i}$
\hfill\Comment{parameter all-reduce}
\For{$i=1,\dots,M$ \textbf{in parallel}}
\State $\tilde{\bd}_t\up{i}\gets\A(\bbx_t,\bbx_{t-1}\up{i},\bd_{t-1}\up{i},\K{avg})$
\hfill\Comment{accumulate and correct drift}
\EndFor
\State $\displaystyle\bd_t\gets\frac{1}{M}\sum_{i=1}^M\tilde{\bd}_t\up{i}$
\hfill\Comment{estimator all-reduce}
\For{$i=1,\dots,M$ \textbf{in parallel}}
\State $\bd_t\up{i}\gets\bd_t, \quad \bx_{t+1}\up{i}\gets\bx_t-\eta_t\bd_t,\quad \bbx_{t+1}\up{i}\gets\gamma_t\bx_{t+1}\up{i}+(1-\gamma_t)\bbx_t, \quad t\gets t+1$
\EndFor
\EndFor
\State \textbf{return} $\bbx_t$
\hfill\Comment{$t=R(\K{loc}+1)$}
\end{algorithmic}
\end{algorithm}

\paragraph{Local MixVR.}
The full procedure is summarized in \cref{alg:local-mu2}.
We now turn to the theoretical analysis of Local MixVR and present its main guarantees.

\begin{lemma}[Local Stochastic Noise]\label{lem:error-bound}
Let $f:\reals^d\to\reals$ be a convex $L$-smooth function with global minimizer $\bxs$ and suppose that \cref{eq:bounded-variance,eq:Main} hold.
Then, applying \cref{alg:local-mu2} with $\alpha=\frac{1}{2},\beta_t=\frac{1}{t},\gamma_t=\frac{2}{t+2},\eta_t=t\cdot\eta$, where assume that the stepsize $\eta$ at every time $t\geq1$ satisfies $\eta\leq\frac{1}{8t}\mini{\frac{1}{L},\frac{1}{\sigmal\sqrt{4\K{loc}+\frac{5t}{M}}}}$, gives
\als
\frac{1}{M}\sum_{i=1}^Mt^2\Exp{\normsq{\beps_t\up{i}}}\leq\underbrace{4\K{loc}\tsigma^2}_{{\textnormal{local error}}}+\underbrace{\frac{5t\tsigma^2}{M}}_{{\textnormal{shared error}}}
\eals
where $\tsigma^2=\sigma^2+32\sigmal^2D_1^2$.
\end{lemma}

\begin{remark}
\Cref{lem:error-bound} highlights a key temporal divide in the stochastic error of Local MixVR.
At synchronization, each worker effectively incorporates information from all workers, so stochasticity before the most recent round is globalized, producing the shared-error term and recovering the optimal $1/M$ minibatch variance reduction.
Only the “raw” stochasticity from the current local window of length $\oo{\K{loc}}$ avoids this reduction.
Thus, as $t$ grows, local noise becomes negligible: nearly all optimization history is globally averaged, while only the most recent $\oo{\K{loc}}$ steps remain local.
\end{remark}

\begin{theorem}\label{thm:final}
Let $f:\reals^d\to\reals$ be a convex $L$-smooth function with global minimizer $\bxs$ and suppose that \cref{eq:bounded-variance,eq:Main} hold.
Then, applying \cref{alg:local-mu2} with $\alpha=\frac{1}{2},\beta_t=\frac{1}{t},\gamma_t=\frac{2}{t+2},\eta_t=t\cdot\eta$, where $\eta=\mini{\frac{1}{8LT},\frac{D_1}{T\sqrt{2\sigma^2+64\sigmal^2D_1^2}\sqrt{4\K{loc}+\frac{5T}{M}}}}$ gives
\als
\Delta_T\leq\oo{\frac{LD_1^2}{KR}+\frac{\tsigma D_1}{K^{1/2}R}+\frac{\tsigma D_1}{\sqrt{MKR}}}
\eals
where $T=R(\K{loc}+1)$ is the final iteration, $\tsigma=\ot{\sigma+\sigmal D_1}$, and $\Delta_T:=\ExpB{\f{\bbx_T}}-\f{\bxs}$, with $\bbx_T:=\frac{1}{M}\sum_{i=1}^M\bbx_T\up{i}$.
\end{theorem}

\begin{remark}
\Cref{thm:final} establishes that Local MixVR achieves the optimal statistical rate while reducing the required communication rounds compared to existing methods (\cref{tab:main-comparison}).
Specifically, in the regime $M\leq\oo{N^{1/4}}$, our bound improves upon the state-of-the-art Minibatch ASGD baseline.
Notably, the communication complexity derived from \cref{thm:final} is entirely independent of the total sample size $N$, scaling only with the number of workers $M$.
This sample-independent guarantee ensures that Local MixVR remains communication-efficient as datasets grow, making it well-suited to common settings with massive datasets.
\end{remark}

\section{Experiments}
We evaluate the effect of the communication budget on image classification using Fashion-MNIST \citep{xiao2017fashion} and CIFAR-10 \citep{krizhevsky2014cifar}.
All methods are implemented in PyTorch and are evaluated using the same data budget at each communication setting.
The Fashion-MNIST experiments were run using an Apple M4 chip, whereas the CIFAR-10 experiments were run on an NVIDIA A100 GPU.
Each point reports the mean final test accuracy over three independent seeds, with error bars denoting one standard deviation.
For Fashion-MNIST, we train a LeNet-style convolutional network with $M\in\{4,8,16\}$ workers, a per-worker batch size of $64$, and a fixed global budget of two epochs.
For CIFAR-10, we train a ResNet-18 \citep{he2016deep} with GroupNorm, initialized from an ImageNet checkpoint, using $M=8$ workers, a per-worker batch size of $16$, and $30$ epochs.
Further experimental results and implementation details are provided in \cref{app:exp}.

We compare Local MixVR with Local SGD, Local Momentum, Minibatch SGD, and Minibatch ASGD.
On Fashion-MNIST, Local MixVR achieves the highest accuracy among the local-update methods for every tested value of $R$ and all three worker counts (\cref{fig:fashion-r-sweep-m4,fig:fashion-r-sweep-m8,fig:fashion-r-sweep-m16}).
Although the absolute accuracy of the local methods decreases as the fixed global data budget is distributed across more workers, the benefit of variance reduction becomes more pronounced.
At $R=24$, the margin of Local MixVR over Local SGD increases from $1.90$ percentage points for $M=4$ to $2.93$ for $M=8$ and $4.03$ for $M=16$; the corresponding margins for Local Momentum are $1.01$, $1.99$, and $2.94$ points.
The minibatch baselines are less sensitive to the worker count, but Minibatch ASGD becomes more competitive at large communication budgets.
The CIFAR-10 results exhibit the same communication-efficient behavior (\cref{fig:cifar10-accuracy-vs-R}).
Local MixVR outperforms all evaluated baselines throughout the reported range and is especially effective when communication is limited, whereas the minibatch methods require more global updates to approach the accuracy of the local methods.
Together, these results support the theoretical prediction that Local MixVR can reduce synchronization frequency without sacrificing optimization quality.

\begin{figure}[t]
\centering
\includegraphics[width=0.7\linewidth]{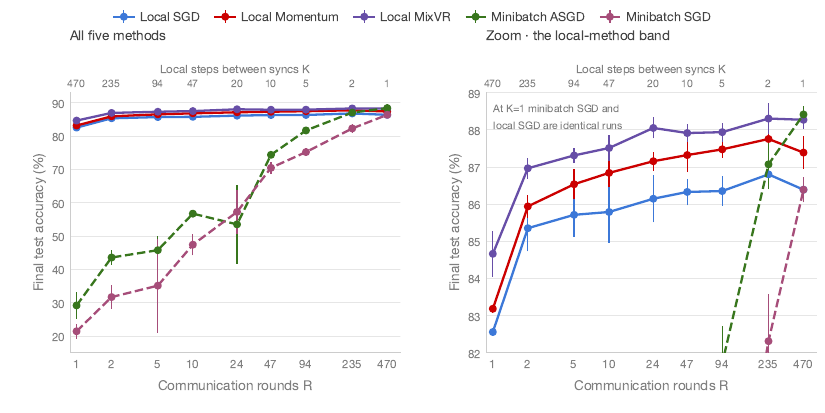}
\caption{\small
\textbf{Test accuracy vs communication rounds $R$ on Fashion-MNIST.
$M=4$ workers.}
The top axis reports the local steps $K=\ceil{470/R}$.
\emph{left:} all five methods; \emph{right:} the local-method accuracy range.}
\label{fig:fashion-r-sweep-m4}
\end{figure}

\section{Conclusions and Future Work}
In this work, we show for the first time that the Minibatch ASGD baseline can be surpassed in common regimes, achieving high performance with reduced communication that is independent of the total number of samples $N$.
By appropriately integrating several variance-reduction techniques, our framework addresses complementary sources of error and effectively mitigates bias accumulated from local updates.
This combination is conceptually natural, and we believe this design principle will prove useful beyond the specific algorithm studied here.
For future work, it would be interesting to derive lower bounds for the distributed setting considered here under the random smooth-function model.
Furthermore, it would be valuable to investigate accelerated variants of our framework, as acceleration has been shown to be a practical and effective approach to enhancing distributed learning with local steps \citep{douillard2023diloco}.

\bibliographystyle{plainnat}
\bibliography{bib.bib}

\clearpage
\appendix
\section{Implications of the Smoothness Assumption}
\label{app:smooth}
We first show that~\cref{eq:Main} implies that~\cref{eq:sigmal} holds for some $\sigmal\in[0,L]$.
\als
&\Exp{\normsq{\left(\df{\bx;\bz}-\df{\bx}\right)-\left(\df{\by;\bz}-\df{\by}\right)}} \\
=&\Exp{\normsq{\df{\bx;\bz}-\df{\by;\bz}}}-\normsq{\df{\bx}-\df{\by}}\leq L^2\normsq{\bx-\by}
\eals
Here, we used the identity $\ExpB{\df{\bx;\bz}-\df{\by;\bz}}=\df{\bx}-\df{\by}$,
together with the identity $\Exp{\normsq{X-\ExpB{X}}}=\Exp{\normsq{X}}-\normsq{\ExpB{X}}$, and finally~\cref{eq:Main}.
Therefore, we conclude that $\sigma_L\in[0,L]$.

\begin{lemma}\label{lem:smooth-functions}
Let $f:\reals^d\to\reals$ be an $L$-smooth function with global minimizer $\bxs$.
Then, for any $\bx\in\reals^d$,
\als
\normsq{\df{\bx}}\leq2L\left(\f{\bx}-\f{\bxs}\right)
\eals
\end{lemma}

\begin{proof}[Proof of~\cref{lem:smooth-functions}]
By $L$-smoothness, for every $\bx,\by\in\reals^d$,
\als
\f{\by}-\f{\bx}\leq\dotprod{\df{\bx}}{\by-\bx}+\frac{L}{2}\normsq{\by-\bx}
\eals
Choosing $\by=\bx-\frac{1}{L}\df{\bx}$ gives
\als
\f{\by}\leq\f{\bx}-\frac{1}{L}\normsq{\df{\bx}}+\frac{1}{2L}\normsq{\df{\bx}}=\f{\bx}-\frac{1}{2L}\normsq{\df{\bx}}
\eals
Using $\f{\by}\geq\f{\bxs}$ and rearranging yields the claim.
\end{proof}

\section{Local MixVR Analysis}
\label{app:anal}
We use $\Sync_t\subseteq\mathbb{N}$ to denote the set of synchronization steps up to iteration $t$.
For each worker $i$, define the local stochastic error and the local optimality gap by
\als
\beps_t^{(i)}:=\begin{cases}
\bd_t^{(i)}-\df{\bbx_t\up{i}}, & t\notin\Sync_T \\
\bd_t^{(i)}-\df{\bbx_t}, & t\in\Sync_T
\end{cases}, \qquad \Delta_t\up{i}:=\begin{cases}
\f{\bbx_t\up{i}}-\f{\bxs}, & t\notin\Sync_T \\
\f{\bbx_t}-\f{\bxs}, & t\in\Sync_T
\end{cases}
\eals
For a more refined analysis, define $\TD_t^2:=2D_1^2+\frac{4\eta^2t}{M}\sum_{i=1}^M\sum_{k=1}^tk^2\Exp{\normsq{\beps_k\up{i}}}$.
Note that $\TD_t^2$ is non-decreasing in $t$.
We will denote $\TD^2:=\TD_T^2$, as it is a bound on all $\TD_t^2$.

\section{Auxiliary Inequalities}
In this section, we collect several auxiliary inequalities that will be used later in the analysis.

\begin{lemma}[\citep{dahan2025stochastic}]\label{lem:Generic}
Let $\{A_t\}_{t\in[T]}$ be a sequence of nonnegative elements and $B\in\reals$, and assume that for any $t\in[T]$:
\als
A_t\leq B+\frac{1}{2T}\sum_{\tau=1}^TA_\tau
\eals
Then the following bound holds:
\als
A_t\leq2B
\eals
\end{lemma}

\begin{proof}[Proof of~\cref{lem:Generic}]
Summing over $t=1,\dots,T$, we obtain $A_{1:T}\leq TB+\frac{1}{2}A_{1:T}$, and hence $A_{1:T}\leq2TB$.
Substituting this estimate back into the original bound yields:
\als
A_t\leq B+\frac{1}{2T}A_{1:T}\leq2B
\eals
\end{proof}

\begin{lemma}\label{lem:average}
Let $f:\reals^d\to\reals$ be convex, and define $\bbx_t:=\frac{1}{M}\sum_{i=1}^M\bbx_t\up{i}$, $\Delta_t:=\f{\bbx_t}-\f{\bxs}$, and $\Delta_t\up{i}:=\f{\bbx_t\up{i}}-\f{\bxs}$.
Then $\Delta_t\leq\frac{1}{M}\sum_{i=1}^M\Delta_t\up{i}$.
\end{lemma}

\begin{proof}[Proof of~\cref{lem:average}]
Since $f$ is convex, Jensen's inequality implies:
\als
\f{\bbx_t}=\f{\frac{1}{M}\sum_{i=1}^M\bbx_t\up{i}}\leq\frac{1}{M}\sum_{i=1}^M\f{\bbx_t\up{i}}
\eals
Subtracting $\f{\bxs}$ from both sides yields $\Delta_t\leq\frac{1}{M}\sum_{i=1}^M\Delta_t\up{i}$.
\end{proof}
For synchronization steps, $\Delta_t\up{i}=\Delta_t$, so the lemma is trivial.

\begin{lemma}[Similar to Theorem 1 in \citet{cutkosky2019anytime}]\label{lem:anytime}
Let $f:\reals^d\to\reals$ be a convex function with global minimizer $\bxs$.
Let $\{\bx_t\}_{t\geq1}\subset\reals^d$ and $\{\bbx_t\}_{t\geq1}\subset\reals^d$ be as in \cref{alg:local-mu2} with $\gamma_t=\frac{2}{t+2}$.
Then the following holds for any $t\geq1$:
\als
&\frac{t(t+1)}{2M}\sum_{i=1}^M\Delta_t\up{i} \\
\leq&\frac{1}{M}\sum_{i=1}^M\left(\sum_{\tau\in[t]\setminus\Sync_t}\tau\dotprod{\df{\bbx_\tau\up{i}}}{\bx_\tau\up{i}-\bxs}+\sum_{\tau\in\Sync_t}\tau\dotprod{\df{\bbx_\tau}}{\bx_\tau\up{i}-\bxs}\right)
\eals
\end{lemma}

\begin{proof}[Proof of~\cref{lem:anytime}]
Following \citet{cutkosky2019anytime}, define $\alpha_t:=t$ and $\alpha_{1:t}:=\sum_{\tau=1}^t\alpha_\tau$.
Then:
\als
\frac{\alpha_{t+1}}{\alpha_{1:t+1}}=\frac{t+1}{\sum_{\tau=1}^{t+1}\tau}=\frac{2}{t+2}=\gamma_t
\eals
Accordingly, in what follows, we use the equivalent representation of $\gamma_t$ in terms of $\{\alpha_t\}_t$.

For a \textbf{non-synchronization step} after the first one, the momentum update step can be written as:
\als
\bbx_t\up{i}=\frac{\alpha_t}{\alpha_{1:t}}\bx_t\up{i}+\left(1-\frac{\alpha_t}{\alpha_{1:t}}\right)\bbx_{t-1}\up{i}
\eals
Similarly, for a \textbf{first step after synchronization}:
\als
\bbx_t\up{i}=\frac{\alpha_t}{\alpha_{1:t}}\bx_t\up{i}+\left(1-\frac{\alpha_t}{\alpha_{1:t}}\right)\bbx_{t-1}
\eals
At a \textbf{synchronization step}:
\als
\bbx_t=\frac{\alpha_t}{\alpha_{1:t}} \bx_t+\left(1-\frac{\alpha_t}{\alpha_{1:t}}\right)\bbx_{t-1}
\eals
From convexity, we know that $\forall\bx,\by\in\reals^d,\f{\bx}-\f{\by}\leq\dotprod{\df{\bx}}{\bx-\by}$.
In a non-synchronization step, the update rule is $\alpha_{1:t}\bbx_t\up{i}=\alpha_{1:t-1}\bbx_{t-1}\up{i}+\alpha_t\bx_t\up{i}$.
We get:
\als
&\alpha_{1:t}\Delta_t\up{i}:=\alpha_{1:t}\left(\f{\bbx_t\up{i}}-\f{\bxs}\right) \\
=&\alpha_{1:t}\f{\bbx_t\up{i}}-\alpha_{1:t-1}\f{\bbx_{t-1}\up{i}}+\alpha_{1:t-1}\f{\bbx_{t-1}\up{i}}-\alpha_{1:t}\f{\bxs} \\
=&\alpha_{1:t-1}\left(\f{\bbx_{t-1}\up{i}}-\f{\bxs}\right)+\alpha_{1:t-1}\left(\f{\bbx_t\up{i}}-\f{\bbx_{t-1}\up{i}}\right)+\alpha_t\left(\f{\bbx_t\up{i}}-\f{\bxs}\right) \\
\leq&\alpha_{1:t-1}\left(\f{\bbx_{t-1}\up{i}}-\f{\bxs}\right)+\alpha_{1:t-1}\dotprod{\df{\bbx_t\up{i}}}{\bbx_t\up{i}-\bbx_{t-1}\up{i}}+\alpha_t\dotprod{\df{\bbx_t\up{i}}}{\bbx_t\up{i}-\bxs} \\
=&\alpha_{1:t-1}\left(\f{\bbx_{t-1}\up{i}}-\f{\bxs}\right)+\dotprod{\df{\bbx_t\up{i}}}{\alpha_{1:t}\bbx_t\up{i}-\alpha_{1:t-1}\bbx_{t-1}\up{i}-\alpha_t\bxs} \\
=&\alpha_{1:t-1}\left(\f{\bbx_{t-1}\up{i}}-\f{\bxs}\right)+\alpha_t\dotprod{\df{\bbx_t\up{i}}}{\bx_t\up{i}-\bxs} \\
=&\alpha_{1:t-1}\Delta_{t-1}\up{i}+\alpha_t\dotprod{\df{\bbx_t\up{i}}}{\bx_t\up{i}-\bxs}
\eals
Similarly, in a synchronization step, the update rule is $\alpha_{1:t}\bbx_t=\alpha_{1:t-1}\bbx_{t-1}+\alpha_t\bx_t$.
We get:
\als
\alpha_{1:t}\Delta_t\up{i}:=&\alpha_{1:t}\left(\f{\bbx_t}-\f{\bxs}\right) \\
\leq&\alpha_{1:t-1}\left(\f{\bbx_{t-1}}-\f{\bxs}\right)+\alpha_t\dotprod{\df{\bbx_t}}{\bx_t-\bxs} \\
\leq&\frac{1}{M}\sum_{i=1}^M\alpha_{1:t-1}\left(\f{\bbx\up{i}_{t-1}}-\f{\bxs}\right)+\alpha_t\dotprod{\df{\bbx_t}}{\bx_t-\bxs} \\
=&\frac{1}{M}\sum_{i=1}^M\alpha_{1:t-1}\Delta_{t-1}\up{i}+\frac{1}{M}\sum_{i=1}^M\alpha_t\dotprod{\df{\bbx_t}}{\bx\up{i}_t-\bxs}
\eals
where the inequality uses \cref{lem:average}.

In a step after synchronization, the update rule is $\alpha_{1:t}\bbx_t\up{i}=\alpha_{1:t-1}\bbx_{t-1}+\alpha_t\bx_t\up{i}$.
Therefore:
\als
\alpha_{1:t}\Delta_t\up{i}\leq&\alpha_{1:t-1}\Delta\up{i}_{t-1}+\alpha_t\dotprod{\df{\bbx_t\up{i}}}{\bx_t\up{i}-\bxs}
\eals
When averaging over $i\in[M]$, both options yield the same result for non-synchronization steps:
\als
\frac{1}{M}\sum_{i=1}^M\alpha_{1:t}\Delta_t\up{i}\leq\frac{1}{M}\sum_{i=1}^M\alpha_{1:t-1}\Delta_{t-1}\up{i}+\frac{1}{M}\sum_{i=1}^M\alpha_t\dotprod{\df{\bbx_t\up{i}}}{\bx_t\up{i}-\bxs}
\eals
For the synchronization steps, we have:
\als
\frac{1}{M}\sum_{i=1}^M\alpha_{1:t}\Delta_t\up{i}\leq\frac{1}{M}\sum_{i=1}^M\alpha_{1:t-1}\Delta_{t-1}\up{i}+\frac{1}{M}\sum_{i=1}^M\alpha_t\dotprod{\df{\bbx_t}}{\bx_t\up{i}-\bxs}
\eals
Unrolling, we get:
\als
\frac{1}{M}\sum_{i=1}^M\alpha_{1:t}\Delta_t\up{i}\leq&\frac{1}{M}\sum_{i=1}^M\sum_{\tau\in[t]\setminus\Sync_t}\alpha_\tau\dotprod{\df{\bbx_\tau\up{i}}}{\bx_\tau\up{i}-\bxs} \\
+&\frac{1}{M}\sum_{i=1}^M\sum_{\tau\in\Sync_t}\alpha_\tau\dotprod{\df{\bbx_\tau}}{\bx_\tau\up{i}-\bxs}
\eals
Finally, since $\alpha_t=t$ and $\alpha_{1:t}=\sum_{\tau=1}^t\tau=\frac{t(t+1)}{2}$, we input these values.
This proves the result.
\end{proof}

\section{Convergence Bounds}

\begin{lemma}\label{lem:Bound}
Let $f:\reals^d\to\reals$ be a convex, $L$-smooth function with global minimizer $\bxs$, and let $\{\bx_t\up{i}\}_{t\geq1,i\in[M]}$ be the iterates generated by~\cref{alg:local-mu2} with $\gamma_t=\frac{2}{t+2},\eta_t=t\cdot\eta$.
Then, for any $T>0$:
\als
&\frac{1}{2\eta M}\sum_{i=1}^M\normsq{\bx_{T+1}\up{i}-\bxs}+\frac{1}{M}\sum_{i=1}^M\sum_{t=1}^Tt\dotprod{\bd_t\up{i}}{\bx_t\up{i}-\bxs} \\
\leq&\frac{D_1^2}{2\eta}+\frac{2L\eta}{M}\sum_{i=1}^M\sum_{t=1}^Tt^2\Delta_t\up{i}+\frac{\eta}{M}\sum_{i=1}^M\sum_{t=1}^Tt^2\normsq{\beps_t\up{i}}
\eals
where $D_1:=\norm{\bbx_1-\bxs}$.
\end{lemma}

\begin{proof}[Proof of~\cref{lem:Bound}]
Fix an arbitrary worker $i\in[M]$.
By the update rule for $t>0$ that is not a synchronization step:
\als
\bx_{t+1}\up{i}=\bx_t\up{i}-\eta t\bd_t\up{i}
\eals
Thus:
\als
&\normsq{\bx_{t+1}\up{i}-\bxs}=\normsq{\bx_t\up{i}-\eta t\bd_t\up{i}-\bxs} \\
=&\normsq{\bx_t\up{i}-\bxs}-2\eta t\dotprod{\bd_t\up{i}}{\bx_t\up{i}-\bxs}+\eta^2t^2\normsq{\bd_t\up{i}}
\eals
Rearranging it yields:
\al\label{eq:rearranged-one-step}
2\eta t\dotprod{\bd_t\up{i}}{\bx_t\up{i}-\bxs}=\normsq{\bx_t\up{i}-\bxs}-\normsq{\bx_{t+1}\up{i}-\bxs}+\eta^2t^2\normsq{\bd_t\up{i}}
\eal
For a synchronization step $t>0$, the update rule is defined as:
\als
\bx_{t+1}\up{i}=\bx_t-\eta t\bd_t\up{i}
\eals
Therefore, similarly to the analysis for~\cref{eq:rearranged-one-step}, we have for a synchronization step $t>0$:
\al\label{eq:rearranged-one-step2}
2\eta t\dotprod{\bd_t\up{i}}{\bx_t-\bxs}=\normsq{\bx_t-\bxs}-\normsq{\bx_{t+1}\up{i}-\bxs}+\eta^2t^2\normsq{\bd_t\up{i}}
\eal
Consider round $r$.
Let $s_r$ and $e_r$ denote its first and last steps, respectively.
At steps $\{s_r\}_{r=1}^R$, the model is synchronized, meaning that $\bx_{s_r}\up{i}=\bx_{s_r}$ for all $i\in[M]$, while $\{e_r\}_{r=1}^R$ are the synchronization steps that follow \cref{eq:rearranged-one-step2}, where $\bd_{e_r}\up{i}=\bd_{e_r}$.

We now sum \cref{eq:rearranged-one-step} over the iterations belonging to round $r$, over $t=s_r,\dots,e_r-1$:
\al\label{eq:single-round-worker}
2\eta\sum_{t=s_r}^{e_r-1}t\dotprod{\bd_t\up{i}}{\bx_t\up{i}-\bxs}=&\sum_{t=s_r}^{e_r-1}\left(\normsq{\bx_t\up{i}-\bxs}-\normsq{\bx_{t+1}\up{i}-\bxs}\right)+\eta^2\sum_{t=s_r}^{e_r-1}t^2\normsq{\bd_t\up{i}} \non
=&\normsq{\bx_{s_r}\up{i}-\bxs}-\normsq{\bx_{e_r}\up{i}-\bxs}+\eta^2\sum_{t=s_r}^{e_r-1}t^2\normsq{\bd_t\up{i}}
\eal
Adding the synchronization step of \cref{eq:rearranged-one-step2} at $t=e_r$ to \cref{eq:single-round-worker}:
\als
&2\eta\left(\sum_{t=s_r}^{e_r-1}t\dotprod{\bd_t\up{i}}{\bx_t\up{i}-\bxs}+e_r\dotprod{\bd_{e_r}\up{i}}{\bx_{e_r}-\bxs}\right) \\
\leq&\normsq{\bx_{s_r}\up{i}-\bxs}-\normsq{\bx_{e_r+1}\up{i}-\bxs}+\normsq{\bx_{e_r}-\bxs}-\normsq{\bx_{e_r}\up{i}-\bxs}+\eta^2\sum_{t=s_r}^{e_r}t^2\normsq{\bd_t\up{i}} \\
=&\normsq{\bx_{s_r}-\bxs}-\normsq{\bx_{s_{r+1}}-\bxs}+\normsq{\bx_{e_r}-\bxs}-\normsq{\bx_{e_r}\up{i}-\bxs}+\eta^2\sum_{t=s_r}^{e_r}t^2\normsq{\bd_t\up{i}}
\eals
where in the final equality we used $e_r+1=s_{r+1}$ and $\bx_{s_r}\up{i}=\bx_{s_r}$.

We average over all $M$ workers and obtain:
\als
&\frac{2\eta}{M}\sum_{i=1}^M\sum_{t=s_r}^{e_r}t\dotprod{\bd_t\up{i}}{\bx_t\up{i}-\bxs}\leq\normsq{\bx_{s_r}-\bxs}-\normsq{\bx_{s_{r+1}}-\bxs}+\frac{\eta^2}{M}\sum_{i=1}^M\sum_{t=s_r}^{e_r}t^2\normsq{\bd_t\up{i}} \\
\eals
where we used $\bx_{e_r}=\frac{1}{M}\sum_{i=1}^M\bx_{e_r}\up{i}$ and $\bd_{e_r}\up{i}=\bd_{e_r}$, which yield $\frac{1}{M}\sum_{i=1}^M\dotprod{\bd_{e_r}\up{i}}{\bx_{e_r}-\bxs}=\frac{1}{M}\sum_{i=1}^M\dotprod{\bd_{e_r}\up{i}}{\bx_{e_r}\up{i}-\bxs}$ and $\normsq{\bx_{e_r}-\bxs}\leq\frac{1}{M}\sum_{i=1}^M\normsq{\bx_{e_r}\up{i}-\bxs}$.

Next, summing over $r=1,\dots,R_T$, where $e_{R_T}$ is the last synchronization iteration in $[T]$, gives:
\als
&\frac{2\eta}{M}\sum_{i=1}^M\sum_{r=1}^{R_T}\sum_{t=s_r}^{e_r}t\dotprod{\bd_t\up{i}}{\bx_t\up{i}-\bxs} \\
\leq&\sum_{r=1}^{R_T}\left(\normsq{\bx_{s_r}-\bxs}-\normsq{\bx_{s_{r+1}}-\bxs}\right)+\frac{\eta^2}{M}\sum_{i=1}^M\sum_{r=1}^{R_T}\sum_{t=s_r}^{e_r}t^2\normsq{\bd_t\up{i}} \\
=&\normsq{\bx_{s_1}-\bxs}-\normsq{\bx_{s_{R_T+1}}-\bxs}+\frac{\eta^2}{M}\sum_{i=1}^M\sum_{r=1}^{R_T}\sum_{t=s_r}^{e_r}t^2\normsq{\bd_t\up{i}}
\eals
The intervals $\{s_r,\dots,e_r\}_{r=1}^{R_T}$ partition $\{1,\dots,T-\tau\}$, and thus $\sum_{r=1}^{R_T}\sum_{t=s_r}^{e_r}(\cdot)=\sum_{t=1}^{T-\tau}(\cdot)$, where $\tau\geq0$ is the number of extra iterations after the last synchronization.

Substituting $s_1=1$ and $s_{R_T+1}=e_{R_T}+1=T-\tau+1$ gives:
\als
\frac{2\eta}{M}\sum_{i=1}^M\sum_{t=1}^{T-\tau}t\dotprod{\bd_t\up{i}}{\bx_t\up{i}-\bxs}\leq&\normsq{\bx_1-\bxs}-\normsq{\bx_{T-\tau+1}-\bxs}+\frac{\eta^2}{M}\sum_{i=1}^M\sum_{t=1}^{T-\tau}t^2\normsq{\bd_t\up{i}}
\eals
Adding the average over $i\in[M]$ of \cref{eq:rearranged-one-step} for $t=T-\tau+1,\ldots,T$ gives:
\als
\frac{2\eta}{M}\sum_{i=1}^M\sum_{t=1}^Tt\dotprod{\bd_t\up{i}}{\bx_t\up{i}-\bxs}\leq&\normsq{\bx_1-\bxs}-\frac{1}{M}\sum_{i=1}^M\normsq{\bx_{T+1}\up{i}-\bxs}+\frac{\eta^2}{M}\sum_{i=1}^M\sum_{t=1}^Tt^2\normsq{\bd_t\up{i}}
\eals
Using $\norm{\bx_1-\bxs}=\norm{\bbx_1-\bxs}=D_1$, and the bound $\normsq{\ba+\bb}\leq2\normsq{\ba}+2\normsq{\bb}$ with the identity:
\als
\bd_t\up{i}-\beps_t\up{i}=\begin{cases}
\df{\bbx_t\up{i}}, & \text{if } t\notin\Sync_T \\
\df{\bbx_t}, & \text{if } t\in\Sync_T
\end{cases}
\eals
yields:
\als
&\frac{1}{M}\sum_{i=1}^M\normsq{\bx_{T+1}\up{i}-\bxs}+\frac{2\eta}{M}\sum_{i=1}^M\sum_{t=1}^Tt\dotprod{\bd_t\up{i}}{\bx_t\up{i}-\bxs} \\
\leq&D_1^2+\frac{2\eta^2}{M}\sum_{i=1}^M\left(\sum_{t\in[T]\setminus\Sync_T}t^2\normsq{\df{\bbx_t\up{i}}}+\sum_{t\in\Sync_T}t^2\normsq{\df{\bbx_t}}\right)+\frac{2\eta^2}{M}\sum_{i=1}^M\sum_{t=1}^Tt^2\normsq{\beps_t\up{i}}
\eals
Dividing by $2\eta$ and using \cref{lem:smooth-functions}, we get:
\als
&\frac{1}{2\eta M}\sum_{i=1}^M\normsq{\bx_{T+1}\up{i}-\bxs}+\frac{1}{M}\sum_{i=1}^M\sum_{t=1}^Tt\dotprod{\bd_t\up{i}}{\bx_t\up{i}-\bxs} \\
\leq&\frac{D_1^2}{2\eta}+\frac{2L\eta}{M}\sum_{i=1}^M\sum_{t=1}^Tt^2\Delta_t\up{i}+\frac{\eta}{M}\sum_{i=1}^M\sum_{t=1}^Tt^2\normsq{\beps_t\up{i}}
\eals
\end{proof}

\begin{lemma}\label{lem:dist}
Let $f:\reals^d\to\reals$ be a convex, $L$-smooth function with global minimizer $\bxs$, and let $\left\{\bx_t\up{i}\right\}_{t\geq1,i\in[M]}$ be the iterates generated by~\cref{alg:local-mu2} with $\gamma_t=\frac{2}{t+2},\eta_t=t\cdot\eta$, where $\eta\leq\frac{1}{8LT}$.
Then, for any $0<t\leq T$:
\als
\normsq{\bx_{t+1}-\bxs}\leq&\frac{1}{M}\sum_{i=1}^M\normsq{\bx_{t+1}\up{i}-\bxs}\leq\TD_t^2 \\
\normsq{\bbx_{t+1}-\bxs}\leq&\frac{1}{M}\sum_{i=1}^M\normsq{\bbx_{t+1}\up{i}-\bxs}\leq\TD_t^2
\eals
where $D_1:=\norm{\bbx_1-\bxs}$ and $\TD_t^2:=2D_1^2+\frac{4\eta^2t}{M}\sum_{i=1}^M\sum_{\tau=1}^t\tau^2\normsq{\beps_\tau\up{i}}$.

In addition,
\als
\frac{1}{M}\sum_{i=1}^M\Delta_t\up{i}\leq\frac{2D_1^2}{\eta t^2}+\frac{4\eta}{tM}\sum_{i=1}^M\sum_{\tau=1}^t\tau^2\normsq{\beps_\tau\up{i}}
\eals
\end{lemma}

\begin{proof}[Proof of~\cref{lem:dist}]
The first inequality in each line follows from $\bx_t=\frac{1}{M}\sum_{i=1}^M\bx_t\up{i}$ and $\bbx_t=\frac{1}{M}\sum_{i=1}^M\bbx_t\up{i}$.
We will focus on the second part.
From \cref{lem:Bound}, we get:
\als
&\frac{1}{2\eta M}\sum_{i=1}^M\normsq{\bx_{t+1}\up{i}-\bxs}\leq\frac{D_1^2}{2\eta}-\frac{1}{M}\sum_{i=1}^M\sum_{\tau=1}^t\tau\dotprod{\bd_\tau\up{i}}{\bx_\tau\up{i}-\bxs} \\
+&\frac{2L\eta}{M}\sum_{i=1}^M\sum_{\tau=1}^t\tau^2\Delta_\tau\up{i}+\frac{\eta}{M}\sum_{i=1}^M\sum_{\tau=1}^t\tau^2\normsq{\beps_\tau\up{i}}
\eals
Using the definition of $\beps_\tau\up{i}$, we split the negative term into two terms:
\als
&\frac{1}{M}\sum_{i=1}^M\sum_{\tau=1}^t\tau\dotprod{\bd_\tau\up{i}}{\bx_\tau\up{i}-\bxs}=\frac{1}{M}\sum_{i=1}^M\sum_{\tau=1}^t\tau\dotprod{\beps_\tau\up{i}}{\bx_\tau\up{i}-\bxs} \\
+&\frac{1}{M}\sum_{i=1}^M\left(\sum_{\tau\in[t]\setminus\Sync_t}\tau\dotprod{\df{\bbx_\tau\up{i}}}{\bx_\tau\up{i}-\bxs}+\sum_{\tau\in\Sync_t}\tau\dotprod{\df{\bbx_\tau}}{\bx_\tau\up{i}-\bxs}\right)
\eals
The first term can be bounded by Young's inequality $2\dotprod{\ba}{\bb}\leq\lambda\normsq{\ba}+\lambda^{-1}\normsq{\bb}$ with $\ba=\tau\beps_\tau\up{i},\bb=\bx_\tau\up{i}-\bxs,\lambda=\eta(2T-1)$:
\als
-\frac{1}{M}\sum_{i=1}^M\sum_{\tau=1}^t\tau\dotprod{\beps_\tau\up{i}}{\bx_\tau\up{i}-\bxs}\leq&\frac{\eta(2T-1)}{2M}\sum_{i=1}^M\sum_{\tau=1}^t\tau^2\normsq{\beps_\tau\up{i}} \\
+&\frac{1}{2\eta(2T-1)M}\sum_{i=1}^M\sum_{\tau=1}^t\normsq{\bx_\tau\up{i}-\bxs}
\eals
The second term can be bounded by \cref{lem:anytime}:
\als
&-\frac{1}{M}\sum_{i=1}^M\left(\sum_{\tau\in[t]\setminus\Sync_t}\tau\dotprod{\df{\bbx_\tau\up{i}}}{\bx_\tau\up{i}-\bxs}+\sum_{\tau\in\Sync_t}\tau\dotprod{\df{\bbx_\tau}}{\bx_\tau\up{i}-\bxs}\right) \\
\leq&-\frac{t(t+1)}{2M}\sum_{i=1}^M\Delta_t\up{i}
\eals
Finally, we get:
\als
&\frac{1}{2M}\sum_{i=1}^Mt(t+1)\Delta_t\up{i}+\frac{1}{2\eta M}\sum_{i=1}^M\normsq{\bx_{t+1}\up{i}-\bxs}\leq\frac{D_1^2}{2\eta}+\frac{\eta(2T+1)}{2M}\sum_{i=1}^M\sum_{\tau=1}^t\tau^2\normsq{\beps_\tau\up{i}} \\
+&\frac{2L\eta}{M}\sum_{i=1}^M\sum_{\tau=1}^t\tau(\tau+1)\Delta_\tau\up{i}+\frac{1}{2\eta(2T-1)M}\sum_{i=1}^M\sum_{\tau=1}^t\normsq{\bx_\tau\up{i}-\bxs}
\eals
Adding $\frac{1}{2\eta(2T-1)M}\sum_{i=1}^M\normsq{\bbx_{t+1}\up{i}-\bxs}$ to both sides, using $D_1:=\norm{\bx_1-\bxs}=\norm{\bx_1\up{i}-\bxs}$ and $2L\eta\leq\frac{1}{4T}$, and extending the sums to $T$, we obtain:
\als
&\frac{1}{2M}\sum_{i=1}^Mt(t+1)\Delta_t\up{i}+\frac{T}{\eta(2T-1)M}\sum_{i=1}^M\normsq{\bx_{t+1}\up{i}-\bxs} \\
\leq&\frac{D_1^2T}{\eta(2T-1)}+\frac{\eta(2T+1)}{2M}\sum_{i=1}^M\sum_{\tau=1}^t\tau^2\normsq{\beps_\tau\up{i}} \\
+&\frac{1}{4TM}\sum_{i=1}^M\sum_{\tau=1}^T\tau(\tau+1)\Delta_\tau\up{i}+\frac{1}{2\eta(2T-1)M}\sum_{i=1}^M\sum_{\tau=1}^T\normsq{\bx_{\tau+1}\up{i}-\bxs}
\eals
We can now apply \cref{lem:Generic} for $t\leq T$:
\als
&\frac{t(t+1)}{2M}\sum_{i=1}^M\Delta_t\up{i}+\frac{T}{\eta(2T-1)M}\sum_{i=1}^M\normsq{\bx_{t+1}\up{i}-\bxs} \\
\leq&\frac{2D_1^2T}{\eta(2T-1)}+\frac{\eta(2T+1)}{M}\sum_{i=1}^M\sum_{\tau=1}^T\tau^2\normsq{\beps_\tau\up{i}}
\eals
Setting $T=t$ and multiplying by $\eta\frac{2t-1}{t}$, we obtain:
\als
&\frac{\eta(t+1)(2t-1)}{2M}\sum_{i=1}^M\Delta_t\up{i}+\frac{1}{M}\sum_{i=1}^M\normsq{\bx_{t+1}\up{i}-\bxs} \\
\leq&2D_1^2+\frac{\eta^2(2t+1)(2t-1)}{tM}\sum_{i=1}^M\sum_{\tau=1}^t\tau^2\normsq{\beps_\tau\up{i}}
\eals
Since $(t+1)(2t-1)=2t^2+t-1\geq2t^2$ and $(2t+1)(2t-1)=4t^2-1\leq4t^2$:
\als
\frac{\eta t^2}{M}\sum_{i=1}^M\Delta_t\up{i}+\frac{1}{M}\sum_{i=1}^M\normsq{\bx_{t+1}\up{i}-\bxs}\leq2D_1^2+\frac{4\eta^2t}{M}\sum_{i=1}^M\sum_{\tau=1}^t\tau^2\normsq{\beps_\tau\up{i}}
\eals
Next, we use $\Delta_t\up{i}\geq0$:
\als
\frac{1}{M}\sum_{i=1}^M\normsq{\bx_{t+1}\up{i}-\bxs}\leq2D_1^2+\frac{4\eta^2t}{M}\sum_{i=1}^M\sum_{\tau=1}^t\tau^2\normsq{\beps_\tau\up{i}}
\eals
The bound on $\bbx_{t+1}\up{i}$ follows from the bound on $\bx_{t+1}\up{i}$ and $\bx_{t+1}$ because $\bbx_{t+1}\up{i}$ is a weighted average of such terms.

Using the nonnegativity of the norm and dividing by $\eta t^2$, we also obtain:
\als
\frac{1}{M}\sum_{i=1}^M\Delta_t\up{i}\leq\frac{2D_1^2}{\eta t^2}+\frac{4\eta}{tM}\sum_{i=1}^M\sum_{\tau=1}^t\tau^2\normsq{\beps_\tau\up{i}}
\eals
\end{proof}

\begin{lemma}\label{lem:diff}
Let $\{\bx_t\}_{t\geq1}\subset\reals^d$ and $\{\bbx_t\}_{t\geq1}\subset\reals^d$ be such that $\bbx_t$ is the $\{\alpha_t\geq0\}_{t\geq1}$-weighted average of $\{\bx_t\}_{t\geq1}$.
Then running~\cref{alg:local-mu2} with $\gamma_t=\frac{2}{t+2},\eta_t=t\cdot\eta$, where $\eta\leq\frac{1}{8LT}$, yields:
\als
\frac{1}{M}\sum_{i=1}^M\normsq{\bbx_t\up{i}-\bbx_t}\leq\frac{4\K{loc}^2\TD_{t-1}^2}{t^2}
\eals
where $D_1:=\norm{\bbx_1-\bxs}$ and $\TD_t^2:=2D_1^2+\frac{4\eta^2t}{M}\sum_{i=1}^M\sum_{\tau=1}^t\tau^2\normsq{\beps_\tau\up{i}}$.
\end{lemma}

\begin{proof}[Proof of~\cref{lem:diff}]
Define $\alpha_t:=t$ and $\alpha_{1:t}:=\sum_{\tau=1}^t\alpha_\tau$; then $\gamma_t=\frac{\alpha_{t+1}}{\alpha_{1:t+1}}=\frac{2}{t+2}$.
Letting $t_0$ denote the last synchronization step, we can write $\alpha_{1:t}\bbx_t\up{i}=\alpha_{1:t_0+1}\bbx_{t_0+1}\up{i}+\sum_{\tau=t_0+2}^t\alpha_\tau\bx_\tau\up{i}$.
Similarly, $\alpha_{1:t}\bbx_t=\alpha_{1:t_0+1}\bbx_{t_0+1}+\sum_{\tau=t_0+2}^t\alpha_\tau\bx_\tau$.
Note that $\bbx_{t_0+1}\up{i}=\bbx_{t_0+1}$.
It follows that:
\als
\bbx_t\up{i}-\bbx_t=\frac{1}{\alpha_{1:t}}\sum_{\tau=t_0+2}^t\alpha_\tau\left(\bx_\tau\up{i}-\bx_\tau\right)=\frac{2}{t(t+1)}\sum_{\tau=t_0+2}^t\tau\left(\bx_\tau\up{i}-\bx_\tau\right)
\eals
Bounding:
\als
\frac{1}{M}\sum_{i=1}^M\normsq{\bbx_t\up{i}-\bbx_t}=&\frac{4}{t^2(t+1)^2M}\sum_{i=1}^M\normsq{\sum_{\tau=t_0+2}^t\tau\left(\bx_\tau\up{i}-\bx_\tau\right)} \\
\leq&\frac{4\left(\sum_{\tau=t_0+2}^t\tau\right)}{t^4M}\sum_{i=1}^M\sum_{\tau=t_0+2}^t\tau\normsq{\bx_\tau\up{i}-\bx_\tau} \\
\leq&\frac{4\left(\sum_{\tau=t_0+2}^t\tau\right)}{t^4M}\sum_{i=1}^M\sum_{\tau=t_0+2}^t\tau\normsq{\bx_\tau\up{i}-\bxs}
\eals
where we used Jensen's inequality and then $\Exp{\normsq{X-\ExpB{X}}}\leq\Exp{\normsq{X}}$, with $X=\bx_\tau\up{i}-\bxs$ and the expectation replaced by averaging over $i$.
We can now use \cref{lem:dist} to obtain:
\als
\frac{1}{M}\sum_{i=1}^M\normsq{\bbx_t\up{i}-\bbx_t}\leq\left(\frac{2\sum_{\tau=t_0+2}^t\tau}{t^2}\right)^2\TD_{t-1}^2
\eals
We bound $\sum_{\tau=t_0+2}^t\tau\leq(t-t_0-1)t$, and $t-t_0-1\leq\K{loc}$.
\als
\frac{1}{M}\sum_{i=1}^M\normsq{\bbx_t\up{i}-\bbx_t}\leq\frac{4\K{loc}^2\TD_{t-1}^2}{t^2}
\eals
\end{proof}

\section{Error Bound}

\begin{lemma}[Stochastic error decomposition]\label{lem:unroll-error}
Assume~\cref{alg:mu2} uses $\beta_t=\frac{1}{t}$.
Fix any $t>0$, and let $t_0$ be the last synchronization step before or at time $t$.
For each worker $i$, define
\als
\bv_\tau\up{i}:=&\begin{cases}
\frac{1}{\K{avg}}\sum_{k=1}^{\K{avg}}\left(\df{\bbx_\tau;\bz_\tau\up{i,k}}-\df{\bbx_\tau}\right), & \tau\in\Sync_T \\
\df{\bbx_\tau\up{i};\bz_\tau\up{i}}-\df{\bbx_\tau\up{i}}, & \tau\notin\Sync_T
\end{cases} \\
\tbv_{\tau-1}\up{i}:=&\begin{cases}
\frac{1}{\K{avg}}\sum_{k=1}^{\K{avg}}\left(\df{\bbx_{\tau-1}\up{i};\bz_\tau\up{i,k}}-\df{\bbx_{\tau-1}\up{i}}\right), & \tau\in\Sync_T \\
\df{\bbx_{\tau-1};\bz_\tau\up{i}}-\df{\bbx_{\tau-1}}, & \tau-1\in\Sync_T \\
\df{\bbx_{\tau-1}\up{i};\bz_\tau\up{i}}-\df{\bbx_{\tau-1}\up{i}}, & \text{otherwise}
\end{cases}
\eals
Then $\{\tau\bv_\tau\up{i}-(\tau-1)\tbv_{\tau-1}\up{i}\}_\tau$ is a martingale difference sequence, and
\als
t^2\Exp{\normsq{\beps_t\up{i}}}=\frac{1}{M^2}\sum_{j=1}^M\sum_{\tau=1}^{t_0}\Exp{\normsq{\tau\bv_\tau\up{j}-(\tau-1)\tbv_{\tau-1}\up{j}}}+\sum_{\tau=t_0+1}^t\Exp{\normsq{\tau\bv_\tau\up{i}-(\tau-1)\tbv_{\tau-1}\up{i}}}
\eals
\end{lemma}

\Cref{lem:unroll-error} establishes that the local stochastic error $\beps_t\up{i}$ integrates noise across \textbf{all preceding iterations}, but its variance is not distributed uniformly.
Specifically, noise generated prior to the most recent synchronization step $t_0$ has been averaged across all $M$ workers, effectively scaling its variance by a factor of $\frac{1}{M}$.
In contrast, the stochasticity accumulated within the current round spans at most $\K{loc}$ iterations and is restricted to the stochasticity of one gradient.

\begin{proof}[Proof of~\cref{lem:unroll-error}]
We examine the recursion of $\beps_t\up{i}$ in all three cases.

\paragraph{Synchronization step $t_0$.}
Let $t_0$ denote the most recent synchronization step before iteration $t$.
At synchronization, all workers share the same estimator, and therefore:
\als
&\bd_{t_0}\up{i}=\bd_{t_0}=(1-\beta_{t_0})\bd_{t_0-1} \\
+&\frac{1}{\K{avg}M}\sum_{j=1}^M\sum_{k=1}^{\K{avg}}\left(\df{\bbx_{t_0};\bz_{t_0}\up{j,k}}-(1-\beta_{t_0})\df{\bbx_{t_0-1}\up{j};\bz_{t_0}\up{j,k}}\right)
\eals
where $\bd_{t_0}:=\frac{1}{M}\sum_{j=1}^M\Tilde{\bd}_{t_0}\up{j}$ and $\bd_{t_0-1}:=\frac{1}{M}\sum_{j=1}^M\bd_{t_0-1}\up{j}$.
Choosing $\beta_t=\frac{1}{t}$ yields:
\als
&t_0\bd_{t_0}\up{i}=\frac{1}{M}\sum_{j=1}^M(t_0-1)\bd_{t_0-1}\up{j} \\
+&\frac{1}{\K{avg}M}\sum_{j=1}^M\sum_{k=1}^{\K{avg}}
\left(t_0\df{\bbx_{t_0};\bz_{t_0}\up{j,k}}-(t_0-1)\df{\bbx_{t_0-1}\up{j};\bz_{t_0}\up{j,k}}\right)
\eals
Subtracting $t_0\df{\bbx_{t_0}}$ from both sides, and adding and subtracting $\frac{1}{M}\sum_{j=1}^M(t_0-1)\df{\bbx_{t_0-1}\up{j}}$, we obtain:
\al\label{eq:local-error-recursion-sync}
t_0\beps_{t_0}\up{i}=\frac{1}{M}\sum_{j=1}^M(t_0-1)\beps_{t_0-1}\up{j}+\frac{1}{M}\sum_{j=1}^M\left(t_0\bv_{t_0}\up{j}-(t_0-1)\tbv_{t_0-1}\up{j}\right)
\eal
where
\als
\bv_{t_0}\up{i}:=&\frac{1}{\K{avg}}\sum_{k=1}^{\K{avg}}\df{\bbx_{t_0};\bz_{t_0}\up{i,k}}-\df{\bbx_{t_0}} \\
\tbv_{t_0-1}\up{i}:=&\frac{1}{\K{avg}}\sum_{k=1}^{\K{avg}}\df{\bbx_{t_0-1}\up{i};\bz_{t_0}\up{i,k}}-\df{\bbx_{t_0-1}\up{i}}
\eals

\paragraph{First step after synchronization: $t_0+1$.}
Applying the same argument for the next step we get:
\als
(t_0+1)\bd_{t_0+1}\up{i}=t_0\bd_{t_0}\up{i}+(t_0+1)\df{\bbx_{t_0+1}\up{i};\bz_{t_0+1}\up{i}}-t_0\df{\bbx_{t_0};\bz_{t_0+1}\up{i}}
\eals
Subtracting $(t_0+1)\df{\bbx_{t_0+1}\up{i}}$ and adding and subtracting $t_0\df{\bbx_{t_0}}$, we obtain:
\al\label{eq:local-error-recursion-first}
(t_0+1)\beps_{t_0+1}\up{i}=t_0\beps_{t_0}\up{i}+\left((t_0+1)\bv_{t_0+1}\up{i}-t_0\tbv_{t_0}\up{i}\right)
\eal
where
\als
\bv_{t_0+1}\up{i}:=\df{\bbx_{t_0+1}\up{i};\bz_{t_0+1}\up{i}}-\df{\bbx_{t_0+1}\up{i}}, \quad \tbv_{t_0}\up{i}:=\df{\bbx_{t_0};\bz_{t_0+1}\up{i}}-\df{\bbx_{t_0}}
\eals

\paragraph{Non-synchronization steps after the first post-synchronization iteration.}
Now consider a non-synchronization step $t$ after the first post-synchronization iteration:
\al\label{eq:local-error-recursion}
t\beps_t\up{i}=(t-1)\beps_{t-1}\up{i}+\left(t\bv_t\up{i}-(t-1)\tbv_{t-1}\up{i}\right)
\eal
Here:
\als
\bv_t\up{i}:=\df{\bbx_t\up{i};\bz_t\up{i}}-\df{\bbx_t\up{i}}, \quad \tbv_{t-1}\up{i}:=\df{\bbx_{t-1}\up{i};\bz_t\up{i}}-\df{\bbx_{t-1}\up{i}}
\eals

Unrolling \cref{eq:local-error-recursion} up to $t_0+2$ and adding \cref{eq:local-error-recursion-first} yields:
\al\label{eq:local-error-recursion-unroll}
t\beps_t\up{i}=t_0\beps_{t_0}\up{i}+\sum_{\tau=t_0+1}^{t}\left(\tau\bv_\tau\up{i}-(\tau-1)\tbv_{\tau-1}\up{i}\right)
\eal
This result is true for all non-synchronization steps.

Using \cref{eq:local-error-recursion-sync} at the next synchronization step $t_1$, together with \cref{eq:local-error-recursion-unroll} at $t=t_1-1$:
\als
t_1\beps_{t_1}\up{i}=\frac{1}{M}\sum_{j=1}^Mt_0\beps_{t_0}\up{j}+\frac{1}{M}\sum_{j=1}^M\sum_{\tau=t_0+1}^{t_1}\left(\tau\bv_\tau\up{j}-(\tau-1)\tbv_{\tau-1}\up{j}\right)
\eals
Because $t_0$ is a synchronization step, $\beps_{t_0}\up{j}=\beps_{t_0}\up{i}$; thus,
\als
t_1\beps_{t_1}\up{i}=t_0\beps_{t_0}\up{i}+\frac{1}{M}\sum_{j=1}^M\sum_{\tau=t_0+1}^{t_1}\left(\tau\bv_\tau\up{j}-(\tau-1)\tbv_{\tau-1}\up{j}\right)
\eals
Unrolling this result over all synchronization steps and substituting it into \cref{eq:local-error-recursion-unroll}, we obtain:
\als
t\beps_t\up{i}=\frac{1}{M}\sum_{j=1}^M\sum_{\tau=1}^{t_0}\left(\tau\bv_\tau\up{j}-(\tau-1)\tbv_{\tau-1}\up{j}\right)+\sum_{\tau=t_0+1}^{t}\left(\tau\bv_\tau\up{i}-(\tau-1)\tbv_{\tau-1}\up{i}\right)
\eals
We can see that this is a martingale difference sequence, and thus we get:
\als
t^2\Exp{\normsq{\beps_t\up{i}}}=\frac{1}{M^2}\sum_{j=1}^M\sum_{\tau=1}^{t_0}\Exp{\normsq{\tau\bv_\tau\up{j}-(\tau-1)\tbv_{\tau-1}\up{j}}}+\sum_{\tau=t_0+1}^{t}\Exp{\normsq{\tau\bv_\tau\up{i}-(\tau-1)\tbv_{\tau-1}\up{i}}}
\eals
\end{proof}

\begin{proof}[Proof of~\cref{lem:error-bound}]
We begin with \cref{lem:unroll-error}.

\paragraph{Bounding the stochastic terms.}
It remains to bound $\Exp{\normsq{\tau\bv_\tau\up{i}-(\tau-1)\tbv_{\tau-1}\up{i}}}$.

We consider three cases separately: (i) the first local step after synchronization, (ii) a later local step within the same round, and (iii) a synchronization step.
In all cases, the same elementary decomposition applies:
\als
\Exp{\normsq{\tau\bv_\tau\up{i}-(\tau-1)\tbv_{\tau-1}\up{i}}}\leq2\underbrace{\Exp{\normsq{\bv_\tau\up{i}}}}_{\text{gradient variance}}+2(\tau-1)^2\underbrace{\Exp{\normsq{\bv_\tau\up{i}-\tbv_{\tau-1}\up{i}}}}_{\text{smoothness variance}}
\eals
The first term is the intrinsic stochastic variance of the gradient estimator, whereas the second term measures the change in the estimator between the two points.

\paragraph{Step 1: bounding the gradient variance term.}
The term $\Exp{\normsq{\bv_\tau\up{i}}}$ is directly bounded by the assumption in~\cref{eq:bounded-variance}.
Summing over iterations gives
\als
\frac{1}{M^2}\sum_{j=1}^M\sum_{\tau=1}^{t_0}\Exp{\normsq{\bv_\tau\up{j}}}+\sum_{\tau=t_0+1}^{t}\Exp{\normsq{\bv_\tau\up{i}}}\leq(t-t_0)\sigma^2+\frac{(r-1)\sigma^2}{\K{avg}M}+\frac{(r-1)\K{loc}\sigma^2}{M}
\eals
Here, the first term corresponds to the current round, whose noise has not yet been averaged.
The second term arises from synchronization steps, where averaging over both workers and the $\K{avg}$ samples yields an additional reduction in variance.
The third term arises from averaging all local updates prior to the last synchronization and across workers.
Using $t_0=(r-1)(\K{loc}+1)$ and $\K{avg}\geq1$, we obtain
\al\label{eq:variance-decomp-intuition}
\frac{1}{M^2}\sum_{j=1}^M\sum_{\tau=1}^{t_0}\Exp{\normsq{\bv_\tau\up{j}}}+\sum_{\tau=t_0+1}^{t}\Exp{\normsq{\bv_\tau\up{i}}}\leq(t-t_0)\sigma^2+\frac{t_0\sigma^2}{M}
\eal
Thus, only a limited amount of variance from the most recent steps $t-t_0\leq\K{loc}$ remains local and unaveraged, while updates from earlier iterations have already been synchronized across workers and reduce their stochastic error by a factor of $M$.

\paragraph{Step 2: bounding the smoothness variance term.}
We now turn to $\Exp{\normsq{\bv_\tau\up{i}-\tbv_{\tau-1}\up{i}}}$.
By the assumption in~\cref{eq:sigmal}, this term is bounded by the distance between the two iterates at which the stochastic gradients are evaluated.
More precisely,
\al\label{eq:case-split-second-term}
\Exp{\normsq{\bv_\tau\up{i}-\tbv_{\tau-1}\up{i}}}\leq\begin{cases}\displaystyle
\frac{\sigmal^2}{\K{avg}}\Exp{\normsq{\bbx_{\tau}-\bbx_{\tau-1}\up{i}}}, & \tau\in\Sync_T \\
\sigmal^2\Exp{\normsq{\bbx_{\tau}\up{i}-\bbx_{\tau-1}}}, & \tau-1\in\Sync_T \\
\sigmal^2\Exp{\normsq{\bbx_\tau\up{i}-\bbx_{\tau-1}\up{i}}}, & \text{otherwise}
\end{cases}
\eal
So in every case, the problem reduces to controlling the distance between iterates at that stage.

For non-synchronization steps $\tau,\tau-1\notin\Sync_T$, the anytime momentum update implies
\als
\bbx_\tau\up{i}-\bbx_{\tau-1}\up{i}=\gamma_{\tau-1}\left(\bx_\tau\up{i}-\bbx_{\tau-1}\up{i}\right)
\eals
Using \cref{lem:dist}, setting $\gamma_t=\frac{2}{t+2}$, and defining $\TD_\tau^2:=2D_1^2+\frac{4\eta^2\tau}{M}\sum_{i=1}^M\sum_{k=1}^\tau k^2\normsq{\beps_k\up{i}}$, we obtain
\al\label{eq:following-iter-anytime}
&\frac{1}{M}\sum_{i=1}^M\Exp{\normsq{\bbx_\tau\up{i}-\bbx_{\tau-1}\up{i}}}=\frac{1}{M}\sum_{i=1}^M\gamma_{\tau-1}^2\Exp{\normsq{\bx_\tau\up{i}-\bbx_{\tau-1}\up{i}}} \non
\leq&\frac{1}{M}\sum_{i=1}^M2\gamma_{\tau-1}^2\left(\Exp{\normsq{\bx_\tau\up{i}-\bxs}}+\Exp{\normsq{\bbx_{\tau-1}\up{i}-\bxs}}\right)\leq\frac{16\TD_{\tau-1}^2}{\tau^2}
\eal
Similarly, for $\tau-1\in\Sync_T$ we have $\frac{1}{M}\sum_{i=1}^M\Exp{\normsq{\bbx_\tau\up{i}-\bbx_{\tau-1}}}\leq\frac{16\TD_{\tau-1}^2}{\tau^2}$.

\paragraph{The synchronization drift.}
The synchronization case is slightly different, as the iterates transition from the local point $\bbx_{t_0-1}\up{i}$ to the averaged global point $\bbx_{t_0}$.
Using $\frac{1}{M}\sum_{i=1}^M\bbx_{t_0-1}\up{i}=\bbx_{t_0-1}$, we decompose:
\als
\frac{1}{M}\sum_{i=1}^M\normsq{\bbx_{t_0}-\bbx_{t_0-1}\up{i}}=\normsq{\bbx_{t_0}-\bbx_{t_0-1}}+\frac{1}{M}\sum_{i=1}^M\normsq{\bbx_{t_0-1}\up{i}-\bbx_{t_0-1}}
\eals
The first term is handled exactly as in~\cref{eq:following-iter-anytime}.

For the second term, we employ \cref{lem:diff}, which gives
\als
\frac{1}{M}\sum_{i=1}^M\normsq{\bbx_{t_0-1}\up{i}-\bbx_{t_0-1}}\leq\frac{4\K{loc}^2\TD_{t_0-2}^2}{(t_0-1)^2}
\eals

Compared with~\cref{eq:following-iter-anytime}, this introduces an additional factor $\K{loc}^2$, reflecting the fact that disagreement can accumulate over an entire local round before being reset by synchronization.
Multiplying each of these equations by $(\tau-1)^2$ gives:
\al\label{eq:smoothness-decomp-intuition}
&\frac{1}{M^2}\sum_{j=1}^M\sum_{\tau=1}^{t_0}(\tau-1)^2\Exp{\normsq{\bv_\tau\up{j}-\tbv_{\tau-1}\up{j}}}+\sum_{\tau=t_0+1}^t(\tau-1)^2\Exp{\normsq{\bv_\tau\up{i}-\tbv_{\tau-1}\up{i}}} \non
\leq&16(t-t_0)\sigmal^2\TD_t^2+\frac{(r-1)\sigmal^2\TD_t^2(4\K{loc}^2+16)}{\K{avg}M}+\frac{16(r-1)\K{loc}\sigmal^2\TD_t^2}{M} \non
\leq&16\sigmal^2\TD_t^2\left((t-t_0)+\frac{5(r-1)(\K{loc}+1)}{4M}\right)=16\sigmal^2\TD_t^2\left((t-t_0)+\frac{5t_0}{4M}\right)
\eal
where in the last inequality we used $\alpha=\frac{1}{2}$ to bound $\frac{\K{loc}^2+4}{4\K{avg}}\leq\frac{\K{loc}+5}{4}$.
We also used the fact that $\TD_\tau$ is nondecreasing in $\tau$ to bound all such terms by $\TD_t$.

Combining \cref{eq:variance-decomp-intuition,eq:smoothness-decomp-intuition}, we obtain the following bound for all steps:
\als
\frac{1}{M}\sum_{i=1}^Mt^2\Exp{\normsq{\beps_t\up{i}}}\leq&2\sigma^2\left((t-t_0)+\frac{t_0}{M}\right)+32\sigmal^2\TD_t^2\left((t-t_0)+\frac{5t_0}{4M}\right) \\
\leq&\left(\sigma^2+16\sigmal^2\TD_t^2\right)\left(2(t-t_0)+\frac{5t_0}{2M}\right)
\eals
Substituting the definition of $\TD_t$ and using $2(t-t_0)+\frac{5t_0}{2M}\leq2\K{loc}+\frac{5t}{2M}$, we obtain:
\als
\frac{1}{M}\sum_{i=1}^Mt^2\Exp{\normsq{\beps_t\up{i}}}\leq\left(\sigma^2+32\sigmal^2D_1^2+\frac{64\sigmal^2\eta^2t}{M}\sum_{i=1}^M\sum_{k=1}^tk^2\Exp{\normsq{\beps_k\up{i}}}\right)\left(2\K{loc}+\frac{5t}{2M}\right) \\
\left(\sigma^2+32\sigmal^2D_1^2+\frac{64\sigmal^2\eta^2T}{M}\sum_{i=1}^M\sum_{k=1}^Tk^2\Exp{\normsq{\beps_k\up{i}}}\right)\left(2\K{loc}+\frac{5T}{2M}\right) 
\eals
Where we increased $t$ into $T$.
Choosing $\eta\leq\frac{1}{8\sigmal T\sqrt{4\K{loc}+\frac{5T}{M}}}$, we can apply \cref{lem:Generic} to obtain:
\als
\frac{1}{M}\sum_{i=1}^Mt^2\Exp{\normsq{\beps_t\up{i}}}\leq\left(\sigma^2+32\sigmal^2D_1^2\right)\left(4\K{loc}+\frac{5T}{M}\right)
\eals
Since this is true  for all $t\leq T$, we substitute $t$ back instead of $T$ for both bounds of $\beps_t,\eta$.
\end{proof}

\section{Loss Bound}

\begin{proof}[Proof of~\cref{thm:final}]
Starting from \cref{lem:dist}, we substitute the result of \cref{lem:error-bound} to obtain:
\al\label{eq:loss-eta-tradeoff}
\Delta_T\leq\frac{2D_1^2}{\eta T^2}+4\eta\left(\sigma^2+32\sigmal^2D_1^2\right)\left(4\K{loc}+\frac{5T}{M}\right)
\eal
We pick $\eta=\mini{\frac{1}{8LT},\frac{D_1}{T\sqrt{2\sigma^2+64\sigmal^2D_1^2}\sqrt{4\K{loc}+\frac{5T}{M}}}}$.
Note that this choice satisfies the conditions of \cref{lem:error-bound}, because $T\geq t$ and $\sqrt{2\sigma^2+64\sigmal^2D_1^2}\geq8\sigmal D_1$.
Using this $\eta$, we get:
\als
\Delta_T\leq&\frac{16LD_1^2}{T}+\frac{4D_1}{T}\sqrt{2\sigma^2+64\sigmal^2D_1^2}\sqrt{4\K{loc}+\frac{5T}{M}} \\
\leq&\frac{16LD_1^2}{T}+4D_1\left(\sqrt{2}\sigma+8\sigmal D_1\right)\left(\frac{2\sqrt{\K{loc}}}{T}+\sqrt{\frac{5}{TM}}\right)
\eals
where we used $\sqrt{a+b}\leq\sqrt{a}+\sqrt{b}$.
Using $\K{loc}\leq\frac{K}{2}$ and $T=R(\K{loc}+1)\geq\frac{KR}{2}$, we get:
\als
\Delta_T\leq\frac{32LD_1^2}{KR}+4D_1\left(\sqrt{2}\sigma+8\sigmal D_1\right)\left(\frac{2\sqrt{2}}{\sqrt{K}R}+\frac{\sqrt{10}}{\sqrt{MKR}}\right)
\eals
Bounding $\sqrt{2}\leq\frac{3}{2}$ and $\sqrt{10}\leq\frac{19}{6}$, we obtain:
\als
\Delta_T\leq\frac{32LD_1^2}{KR}+D_1\left(3\sigma+16\sigmal D_1\right)\left(\frac{6}{\sqrt{K}R}+\frac{19}{3\sqrt{MKR}}\right)
\eals
\end{proof}

\clearpage
\section{Experimental Details}
\label{app:exp}

We evaluate Local MixVR against four baselines on Fashion-MNIST and CIFAR-10.
In every experiment, we hold the total sample budget fixed while varying the communication budget. We publicly release the code used for our experiments.\footnote{\url{https://github.com/dahan198/local-mixvr}}

\subsection{Datasets and Architectures}
\label{app:exp-data}

\paragraph{Fashion-MNIST.}
Fashion-MNIST \citep{xiao2017fashion} contains 60,000 training and 10,000 test grayscale $1\times28\times28$ images from 10 clothing categories.
We train a LeNet-style network with $M\in\{4,8,16\}$ workers, a per-worker batch size of $64$, and $2$ epochs.
The network has two $5\times5$ convolutional layers with $20$ and $50$ output channels, each followed by ReLU and $2\times2$ max pooling, and two fully connected layers of sizes $800\to50$ and $50\to10$.

\paragraph{CIFAR-10.}
CIFAR-10 \citep{krizhevsky2014cifar} contains 50,000 training and 10,000 test color $3\times32\times32$ images from 10 classes.
We train a ResNet-18 \citep{he2016deep} with $M=8$ workers, a per-worker batch size of $16$, and $30$ epochs.
Interior convolutions are initialized from the torchvision ImageNet checkpoint.
The $3\times3$ CIFAR stem and classifier have incompatible shapes and are initialized randomly.

\subsection{Methods Compared}
\label{app:exp-methods}
\begin{itemize}
\item
\textbf{Local SGD.}
Workers use plain SGD locally, and the server averages their parameters at every synchronization.
The local optimizer state is reinitialized after each round.
\item
\textbf{Local Momentum.}
Workers use SGD with momentum $0.9$ and dampening $0.9$, corresponding to the exponential-moving-average form of the momentum recursion.
The momentum coefficient $0.9$ is a common parameter choice and corresponds to $\beta_t=0.1$ in our notation, since the previous direction is weighted by $1-\beta_t$.
Momentum buffers are averaged across workers at synchronization and carried into the next interval.
\item
\textbf{Local MixVR.}
We set $\beta_t=0.1$ and $\gamma_t=0.9$.
Each round divides its budget of $K$ samples into $\K{avg}=\ceil{\alpha K}$ accumulation samples and $K-\K{avg}$ local steps; synchronization of the method state is intrinsic to Local MixVR.
\item
\textbf{Minibatch SGD.}
The method performs one optimizer update per round using the same samples as the corresponding local interval, with gradients accumulated in bounded microbatches.
\item
\textbf{Minibatch ASGD.}
This is Minibatch SGD with a Nesterov server optimizer and server momentum $0.9$, a common parameter choice that likewise corresponds to $\beta_t=0.1$ in our notation.
\end{itemize}

\subsection{Tuning and Evaluation}
\label{app:exp-tuning}

Each method is tuned independently at every communication budget.
Every reported point is the configuration with the highest mean final test accuracy for that $(\text{method},R)$ cell.
All configurations are evaluated using three seeds, and we report the mean and population standard deviation of final test accuracy.
At each communication budget, we sweep the learning rate over
$\eta\in\{0.01,0.03,0.06,0.1, 0.3\}$ for every method.
For Local MixVR, we additionally sweep the mixing parameter over
$\alpha\in\{0.01,0.05,0.1\}$.
This parameter controls the division of each round's $K$ samples between the accumulation phase, which uses $\K{avg}=\ceil{\alpha K}$ samples, and the local-update phase, which uses the remaining $K-\K{avg}$ samples.

\subsection{Additional Results: CIFAR-10}
\begin{figure}[H]
\centering
\includegraphics[width=\textwidth]{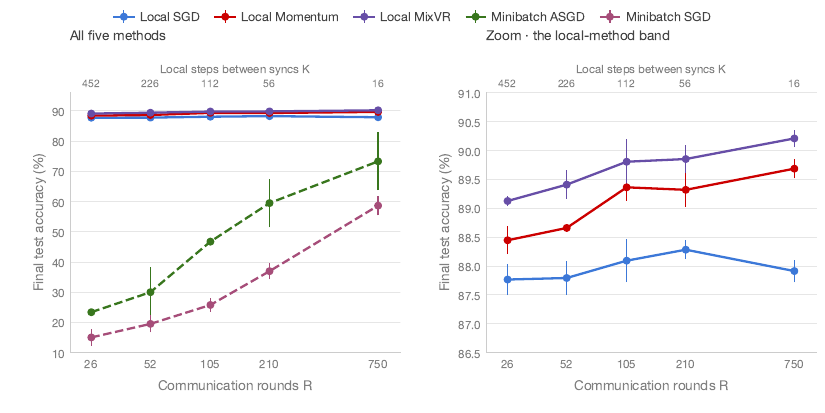}
\caption{\small
\textbf{Test accuracy vs communication rounds $R$ on CIFAR-10.
$M=8$ workers.}
The top axis reports the local steps $K=\ceil{11730/R}$.
\emph{left:} all five methods; \emph{right:} the local-method accuracy range.
Local MixVR achieves the highest accuracy among the local methods at every communication budget, outperforming Local Momentum by $0.45$--$0.75$ percentage points, while Minibatch SGD improves sharply as the number of communication rounds increases.}
\label{fig:cifar10-accuracy-vs-R}
\end{figure}

\subsection{Additional Results: Fashion-MNIST}

\begin{figure}[H]
\centering
\includegraphics[width=\linewidth]{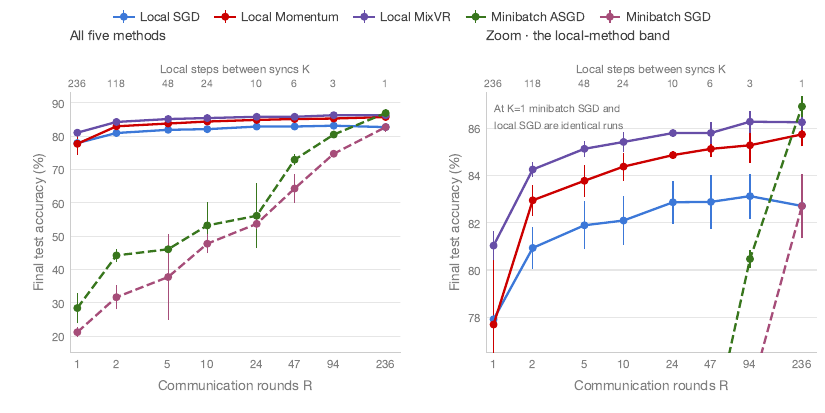}
\caption{\small
\textbf{Test accuracy vs communication rounds $R$ on Fashion-MNIST.
$M=8$ workers.}
The top axis reports the local steps $K=\ceil{236/R}$.
\emph{left:} all five methods; \emph{right:} the local-method accuracy range.
Local MixVR outperforms Local Momentum and Local SGD throughout the sweep; at $R=24$, it exceeds Local SGD by $2.93$ percentage points.}
\label{fig:fashion-r-sweep-m8}
\end{figure}

\begin{figure}[H]
\centering
\includegraphics[width=\linewidth]{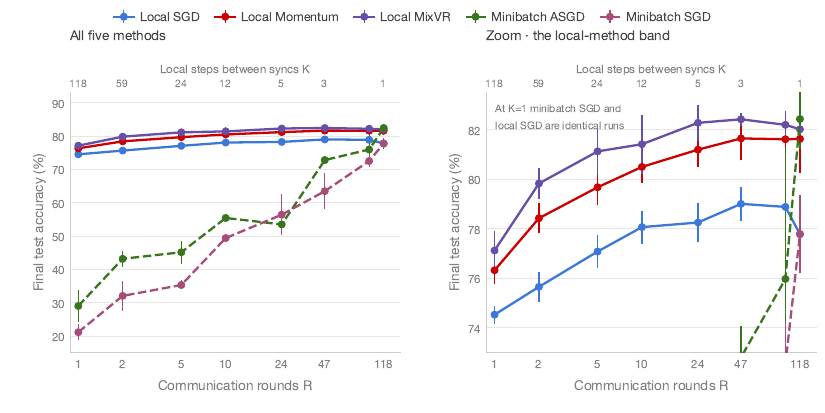}
\caption{\small
\textbf{Test accuracy vs communication rounds $R$ on Fashion-MNIST.
$M=16$ workers.}
The top axis reports the local steps $K=\ceil{118/R}$.
\emph{left:} all five methods; \emph{right:} the local-method accuracy range.
Local MixVR remains the strongest method throughout the sweep; at $R=24$, it exceeds Local SGD by $4.03$ percentage points, while Local Momentum exceeds Local SGD by $2.94$ points.}
\label{fig:fashion-r-sweep-m16}
\end{figure}

\end{document}